\documentclass[11pt]{article}
\usepackage[left=1in, right=1in, top=1in]{geometry}
\usepackage{amsmath}
\usepackage{mathtools}
\usepackage{amsthm}
\usepackage{amsfonts}
\usepackage{amssymb}
\usepackage{amssymb,dsfont,bbm}
\usepackage{mathrsfs}
\usepackage{bm}

\newtheorem{theorem}{Theorem}

\usepackage{xargs}
\usepackage{enumitem} 
\usepackage[numbers]{natbib}
\usepackage{fancyhdr}
\usepackage{setspace}
\usepackage{lastpage}
\usepackage{upgreek}
\usepackage{xcolor}
\usepackage{booktabs}
\usepackage{algorithm}
\usepackage{algorithmic}
\usepackage{subcaption}
\usepackage{multirow}
\usepackage{comment}
\usepackage[utf8]{inputenc}
\usepackage[T1]{fontenc}
\usepackage{url}
\usepackage{nicefrac}
\usepackage{microtype}
\usepackage{tikz-cd}
\usepackage{srcltx}
\usepackage{etoolbox}
\usepackage{nameref}
\usepackage{autonum}

\setcitestyle{number}

\usepackage[colorlinks=true,breaklinks=true,bookmarks=true,urlcolor=blue,citecolor=blue,linkcolor=blue,bookmarksopen=false,draft=false]{hyperref}

\usepackage{aliascnt}
\usepackage{cleveref}

\renewcommand{\algorithmicrequire}{\textbf{Input:}}
\renewcommand{\algorithmicensure}{\textbf{Output:}}

\newcommand{\PE}{\mathbb{E}}

\newcommand{\PP}{\mathbb{P}}
\newcommandx{\genericb}[1][1=]{b_{#1}}
\newcommandx{\Constros}[1][1=]{\operatorname{C}_{\operatorname{Ros},#1}}
\newcommandx{\Constburk}[1][1=]{\operatorname{C}_{\operatorname{Burk}}}
\newcommandx{\driftW}[1][1=]{W_{#1}}

\newcommandx{\metricd}[1][1=]{\mathsf{d}_{#1}}

\newcommandx\invmeasure[1][1=]{\Pi_{#1}}
\newcommandx{\PPjoint}[1][1=]{\PP^{\MKjoint[#1]}}
\newcommandx{\PEjoint}[1][1=]{\PE^{\MKjoint[#1]}}
\newcommandx{\PEMID}[1][1=\alpha]{\PE^{\MK[#1]}}
\newcommandx{\PPMID}[1][1=\alpha]{\PP^{\MK[#1]}}

\newcommandx{\MKjoint}[1][1=]{\bar{\operatorname{P}}_{#1}}
\newcommandx\costw[1][1=]{\mathsf{c}_{#1}}

\newcommandx\Intergrdist[1][1=]{\mathbb{M}_{1}(#1)}

\newcommandx{\mmarkov}[1][1=0]{m^{(\Markov)}_{#1}}

\def\rset{\mathbb{R}}

\def\nset{\ensuremath{\mathbb{N}}}
\def\nsets{\ensuremath{\mathbb{N}^*}}

\newcommandx\sequence[4][2=,3=,4=]
{\ifthenelse{\equal{#3}{}}{\ensuremath{\{ #1_{#2 #4}\}}}{\ensuremath{\{ #1_{#2 #4} \}_{#2 \in #3}}}}

\newcommandx\sequenceD[2][2=]
{\ifthenelse{\equal{#2}{}}{\ensuremath{\{ #1\}}}{\ensuremath{\{ #1\!~:\!~#2  \}}}}
\newcommandx\sequenceDouble[4][3=,4=]
{\ifthenelse{\equal{#3}{}}{\ensuremath{\{ (#1_{#3},#2_{#3})\}}}{\ensuremath{\{ (#1_{#3},#2_{#3})\}_{#3 \in #4}}}}

\newcommandx{\sequencen}[2][2=n\in\nset]{\ensuremath{\{ #1, \eqsp #2 \}}}
\newcommandx\sequencens[2][2=n]
{\ensuremath{\{ #1_{#2} \!~:\!~#2\in\nsets\}}}
\newcommandx\sequencet[4]
{\ensuremath{\{ #1{#2_{#3}} \, : \, \eqsp #3 \in #4 \}}}
\def\PE{\mathbb{E}}

\newcommandx{\PVar}[1][1=]{\ensuremath{\operatorname{Var}_{#1}}}
\newcommandx\conststab[1][1=p]{\varkappa_{#1}}

\newcommandx{\MK}[1][1=\alpha]{\mathrm{P}_{#1}}
\newcommandx\MKK[1][1=\alpha]{\mathrm{K}_{#1}}

\newcommandx{\PEtilde}[1][1=]{\PE^{\mathrm{K}_{#1}}}
\newcommandx{\PPtilde}[1][1=]{\PP^{\mathrm{K}_{#1}}}

\newcommandx{\norm}[2][2=]{\Vert#1 \Vert_{{#2}}}
\newcommandx{\normLigne}[2][2=]{\Vert#1 \Vert_{{#2}}}
\newcommandx{\normLine}[2][2=]{\Vert#1 \Vert_{{#2}}}
\newcommandx{\normop}[2][2=]{\Vert{#1}\Vert_{{#2}}}
\newcommandx{\normopLigne}[2][2=]{\Vert{#1}\Vert_{{#2}}}
\newcommandx{\normopLine}[2][2=]{\Vert{#1}\Vert_{{#2}}}
\newcommandx{\osc}[2][1=]{\mathrm{osc}_{#1}(#2)}

\newcommandx{\normlip}[2][2=\operatorname{Lip}]{\Vert#1 \Vert_{{#2}}}
\newcommand{\lip}{\operatorname{L}}
\newcommandx{\lipspace}[1]{\lip_{#1}}

\newcommandx{\CPP}[3][1=]
{\ifthenelse{\equal{#1}{}}{{\mathbb P}\left(\left. #2 \, \right| #3 \right)}{{\mathbb P}_{#1}\left(\left. #2 \, \right | #3 \right)}}
\newcommandx{\CPPtilde}[3][1=]
{\ifthenelse{\equal{#1}{}}{{\tilde{\mathbb P}}\left(\left. #2 \, \right| #3 \right)}{{\tilde{\mathbb P}}_{#1}\left(\left. #2 \, \right | #3 \right)}}

\newcommandx{\as}[1][1=\PP]{\ensuremath{#1\, -\mathrm{a.s.}}}

\newcommand{\eqsp}{\;}

\newcommandx{\boundmetric}[1][1=]{\kappa_{\MKK[#1]}}

\newcommandx{\Nnorm}[2][1=V]{[ #2]_{#1}}
\newcommandx{\lipnorm}[2][1=g]{[ #1]_{#2}}

\newcommandx{\CPE}[3][1=]{{\mathbb E}^{#3}_{#1}\left[#2\right]}
\newcommandx{\CPEext}[3][1=]{\tilde{\mathbb E}^{#3}_{#1}\left[#2\right]}
\newcommandx{\CPEtilde}[3][1=]{{\tilde{\mathbb E}}^{#3}_{#1}\left[#2\right]}
\newcommandx{\CPEs}[3][1=]{{\mathbb E}^{#3}_{#1}[#2]}

\newcommand{\rmd}{\mathrm{d}}
\def\funcAw{\mathbf{A}}

\def\funcbw{\mathbf{b}}

\newcommandx{\zmfuncA}[2][1=]{\tilde{\funcAw}^{#1}(#2)}
\newcommandx{\zmfuncAw}[1][1=]{\tilde{\funcAw}_{#1}}
\newcommandx{\zmfuncb}[2][1=]{\tilde{\funcbw}^{#1}(#2)}

\newcommandx{\funcct}[2][1=]{\funcctilde^{#1}(#2)}

\newcommandx{\CovC}[1][1=u]{\operatorname{C}_{#1}}

\DeclareMathAlphabet{\mathpzc}{OT1}{pzc}{m}{it}

\def\lyapW{\mathpzc{W}}

\newcommandx{\bias}[1][1=\alpha]{\operatorname{B}_{#1}}

\newcommandx\probaMarkovTilde[2][2=]
{\ifthenelse{\equal{#2}{}}{{\widetilde{\mathbb{P}}_{#1}}}{\widetilde{\mathbb{P}}_{#1}\left[ #2\right]}}

\def\transpose{\top}

\def\funcctilde{\tilde{c}_u}

\newcommandx{\driftb}[1][1=p]{\bar{b}_{#1}}

\def\transpose{\top}

\newcommandx{\boldb}[1][1={q}]{\mathsf{b}_{#1}}

\newcommandx{\ConstGW}[1][1={n,\lyapW}]{\operatorname{G}_{#1}}

\newcommandx{\ConstMW}[1][1={n,\lyapW}]{\operatorname{M}_{#1}}

\Crefname{assumTD}{\textbf{TD}\hspace{-1pt}}{\textbf{TD}\hspace{-1pt}}
\crefname{assumTD}{\textbf{TD}}{\textbf{TD}}

\Crefname{assumptionC}{\textbf{C}\hspace{-1pt}}{\textbf{C}\hspace{-1pt}}
\crefname{assumptionC}{\textbf{C}}{\textbf{C}}

\Crefname{assumptionM}{\textbf{UGE}\hspace{-1pt}}{\textbf{UGE}\hspace{-1pt}}
\crefname{assumptionM}{\textbf{UGE}}{\textbf{UGE}}

\newtheorem{remark}{\textbf{Remark}\hspace{-1pt}}
\Crefname{remark}{\textbf{Remark}\hspace{-1pt}}{\textbf{remark}\hspace{-1pt}}
\crefname{remark}{\textbf{Remark}}{\textbf{remark}}

\def\distance{\mathsf{d}}

\newcommandx{\vartconstwas}[1][1=V]{c_{#1}}

\newcommandx{\deltawas}[1][1=*]{\delta_{#1}}

\newcommandx{\wasser}[4][1=\distance,4=]{\mathbf{W}_{#1}^{#4}\left(#2,#3\right)}
\newcommandx{\covcoeff}[2]{\rho_{#1}^{(#2)}}

\newcommand{\dobrush}{\mathsf{\Delta}}
\newcommandx{\dobru}[3][1=,3=]{\dobrush_{#1}^{#3}( #2)}  %%% dobrushin coefficient

\def\Markov{\mathrm{M}}

\newcommandx{\dlim}[1]{\ensuremath{\stackrel{#1}{\Longrightarrow}}}

\newcommand{\linecomment}[1]{\quad \texttt{//}~#1}

\usepackage[utf8]{inputenc}
\usepackage[T1]{fontenc}
\usepackage{url}
\usepackage{nicefrac}
\usepackage{microtype}
\usepackage{tikz-cd}

\usepackage{srcltx}
\usepackage{etoolbox}
\usepackage{nameref}

\usepackage{autonum}

\begin{document}

\thispagestyle{fancy}
\fancyhf{}
\renewcommand{\headrulewidth}{0pt}
\fancyfoot[L]{\textit{Preprint. Under review.}}
\fancyfoot[C]{}

\begin{center}
  \vspace*{0.5cm}
  
  \noindent\rule{\textwidth}{1pt}
  
  \vspace{4mm}
  
  {\Large\bfseries Scalable LinUCB: Low-Rank Design Matrix Updates \\ for Recommenders with Large Action Spaces}
  
  \vspace{4mm}
  
  \noindent\rule{\textwidth}{1.5pt}
  
  \vspace{1cm}
  
  \begin{tabular}{cc}
    \begin{minipage}{0.4\textwidth}
      \centering
      \textbf{Evgenia Shustova} \\
      HSE University \\
      \texttt{ekshustova@hse.ru}
    \end{minipage}
    &
    \begin{minipage}{0.4\textwidth}
      \centering
      \textbf{Marina Sheshukova} \\
      HSE University \\
      \texttt{msheshukova@hse.ru}
    \end{minipage}
  \end{tabular}
  
  \vspace{0.8cm}
  
  \begin{tabular}{cc}
    \begin{minipage}{0.4\textwidth}
      \centering
      \textbf{Sergey Samsonov} \\
      HSE University \\
      \texttt{svsamsonov@hse.ru}
    \end{minipage}
    &
    \begin{minipage}{0.4\textwidth}
      \centering
      \textbf{Evgeny Frolov} \\
      Personalization Technologies; \\
      HSE University \\
      \texttt{e.frolov@hse.ru}
    \end{minipage}
  \end{tabular}
  
  \vspace{1cm}
  
  \textbf{Abstract}
  
\end{center}

\vspace{3mm}

In this paper, we introduce PSI-LinUCB, a scalable variant of LinUCB that enables efficient training, inference, and memory usage by representing the inverse regularized design matrix as a sum of a diagonal matrix and low-rank correction. We derive numerically stable rank-1 and batched updates that maintain the inverse without explicitly forming the matrix. To control memory growth, we employ a projector-splitting integrator for dynamical low-rank approximation, yielding an average per-step update cost and memory usage of $\mathcal{O}(dr)$ for approximation rank $r$. The inference complexity of the proposed algorithm is $\mathcal{O}(dr)$ per action evaluation. Experiments on recommender system datasets demonstrate the effectiveness of our algorithm.
\vspace{3mm}

\noindent\textbf{Keywords:} LinUCB algorithm, Low-rank approximation, Projector-Splitting integrator

\section{Introduction and Problem Setting}
\label{sec:intro}
% Interactive recommender systems must learn from sparse and delayed feedback while responding within tight latency budgets. Contextual bandits are a popular solution in this setting, as they enable online adaptation and explicitly balance exploration and exploitation \cite{auer2002finite}. These methods are well studied theoretically \cite{auer2002finite,abbasi2011improved,agrawal2013thompson} and are widely used in real-world recommender systems \cite{li2010contextual,yi2023online}.

% Within this class of methods, one of the most popular algorithms is  \emph{LinUCB} \cite{li2010contextual}. LinUCB models the expected reward of an action as a linear function of a $d$-dimensional context vector. The algorithm has well-studied regret bounds \cite{abbasi2011improved} and admits practical implementations for online recommendation settings \cite{li2010contextual}. Together, these properties make LinUCB a reasonable baseline when a system requires online updates.

% LinUCB \cite{li2010contextual} follows the Upper Confidence Bound (UCB) principle, implementing optimism in the face of uncertainty \cite{auer2002finite}.
%  Its theoretical properties have been extensively studied in the literature \cite{abbasi2011improved,chu2011contextual,valko2013finite,lattimore2020bandit}.

Contextual bandits are essential in modern decision making, as they enable online adaptation to dynamic environments and explicitly balance exploration and exploitation \cite{auer2002finite}. These methods are well studied theoretically \cite{auer2002finite,abbasi2011improved,agrawal2013thompson} and are widely used in practice, in particular, in the real-world recommender systems \cite{li2010contextual,yi2023online}. Within this class of methods, LinUCB \cite{li2010contextual,abbasi2011improved} is one of the most commonly used algorithms. It models the expected reward of an action as a linear function of a $d$-dimensional context vector and follows the Upper Confidence Bound (UCB) principle, implementing optimism in the face of uncertainty \cite{auer2002finite}. LinUCB admits practical implementations for online recommendation settings \cite{li2010contextual} and serves as a natural baseline when online updates are required.
\par
For the most part of this paper, we focus on the disjoint LinUCB parametrization \cite{li2010contextual,das2024linear}, yet the techniques that we develop further naturally generalize to hybrid and shared parametrizations used in this algorithm. Formally, at each time step $t \in \{1,\ldots,T\}$, agent observes a set of arms (actions) $\mathcal{A}_t$, where each arm $a \in \mathcal{A}$ is associated with a context vector ${x}_{t,a} \in \rset^d$. LinUCB assumes a linear reward model,
\begin{equation}
\textstyle
\PE[r_{t,a}|{x}_{t,a}] = {x}_{t,a}^\top {\theta}_a^*,
\end{equation}
where ${\theta}_a^* \in \rset^d$ is the unknown true parameter vector for arm $a$. For each arm $a \in \mathcal{A}$, the algorithm maintains a ridge regression estimator with a design matrix
\begin{equation}
\label{eq:design_matrix_lin_ucb}
\textstyle
A_{t,a} = \lambda I + \sum_{s=1}^{t-1} {x}_{s,a} {x}_{s,a}^\top,
\end{equation}
reward vector
\begin{equation}
\textstyle
    b_{t,a} = \sum_{s=1}^{t-1} r_{s,a} {x}_{s,a},
\end{equation}
and parameter estimate
\begin{equation}
\textstyle
\label{eq:param_update_lin_ucb}
\hat{{\theta}}_{a} = A_{t,a}^{-1} b_{t,a}.
\end{equation}
At time $t$, the algorithm selects the arm according to
\begin{equation}
\label{eq:arm_choice_lin_ucb}
\textstyle
a_t = \arg\max_{a \in \mathcal{A}_t} 
\left( 
\hat{{\theta}}_a^\top {x}_{t,a} + \alpha \sqrt{{x}_{t,a}^\top A_{t, a}^{-1} {x}_{t,a}} 
\right),
\end{equation}
where $\alpha \in \rset$ controls the exploration–exploitation trade-off. In online learning problems $\alpha > 0$, but sometimes learning problems with fixed dataset require setting negative values of $\alpha$ (the so-called anti-exploration), see e.g. \cite{rezaeifar2022offline,xu2024provably}.

 LinUCB can also be extended to batch-update settings, where model parameters are updated after multiple interactions (see \Cref{sec:alg_linUCB_batch} in Appendix). At the same time, LinUCB algorithm has several scalability challenges:
\begin{itemize}[noitemsep,topsep=0pt,leftmargin=0em]
    \item \textit{Matrix inversion:}  The complexity of computing  $A_{t, a}^{-1}$ is $O(d^3)$, which makes the algorithm computationally expensive if context dimension $d$ is large.
    \item \textit{Large action space:} The need to store and update a separate $d \times d$ matrix for each action leads to increased time and memory requirements as the number of actions grows.
\end{itemize}

 % Existing approaches mitigate these issues via matrix sketching \cite{kuzborskij2019efficient,chen2020efficient,wen2024matrix}, applying rank-1 Sherman–Morrison updates  \cite{Angioli2025Efficient}, feature projections \cite{yu2017cbrap}, or using diagonal/block-diagonal approximations of the design matrix \cite{yi2023online}.  Each option involves a different trade-off between accuracy, robustness, and computational complexity. We provide a detailed overview of the existing literature in the Appendix, \Cref{sec:related-work}.

Existing work addresses the scalability limitations of LinUCB in several ways. The simplest acceleration relies on exact rank-one updates of the inverse design matrix using the Sherman–Morrison formula \cite{Angioli2025Efficient}, which avoids full matrix inversion but still requires maintaining dense $d \times d$ matrices and does not scale to large feature dimensions and large action space.
Another line of work focuses on approximating the design matrix $\sum_{s=1}^t x_{s,a} x_{s,a}^\top$ with a low-rank representation using matrix sketching techniques \cite{kuzborskij2019efficient,chen2020efficient,wen2024matrix}. While these methods significantly reduce memory and computation costs, they typically process observations sequentially and do not naturally support vectorized or batch updates.
An alternative approach applies random feature projections to reduce the dimensionality of the context vectors before learning \cite{yu2017cbrap}. Such methods enable faster updates but often require a relatively large projected dimension to preserve recommendation quality, making them less efficient than sketching-based low-rank approximations.
Finally, diagonal or block-diagonal approximations of the design matrix have been proposed to reduce computational and memory costs \cite{yi2023online}. We provide a detailed overview of the existing literature in the Appendix, \Cref{sec:related-work}.

\textbf{Our contrubutions.} In this paper, we propose a scalable variant of LinUCB designed for large-scale contexts settings. Our approach is based on a suitable representation of the inverse regularized design matrix. Specifically, we approximate $A_{t,a}^{-1}$ as the sum of a diagonal matrix and a low-rank correction. This representation avoids explicit matrix inversion, and significantly reduces memory usage by storing only the diagonal terms and low-rank factors. Our primary contributions are as follows:

\begin{itemize}[noitemsep,topsep=0pt,leftmargin=0em]
% \item We introduce PSI-LinUCB, a scalable LinUCB variant that maintains a dynamically updated Cholesky-style low-rank factorization of the inverse regularized design matrix. Using a projector-splitting method for dynamical low-rank approximation \cite{lubich2014projector}, the proposed algorithm avoids explicit matrix inversion and achieves an average per-interaction update cost of $\mathcal{O}(dr)$, where $r$ is the approximation rank. The inference complexity is $\mathcal{O}(dr)$ per action.
% \item We empirically evaluate PSI-LinUCB on the MovieLens-1M dataset and several large-scale datasets from the Amazon collection, where standard LinUCB immediately fails even on top-grade hardware due to scalability issues. Our method achieves up to an $8.8\times$ reduction in training and inference time and a $15.9\times$ reduction in memory consumption compared to the exact Sherman–Morrison implementation \cite{Angioli2025Efficient}, while the quality metric (hit rate) at the level of the vanilla LinUCB implementation.
\item We introduce PSI-LinUCB, a scalable LinUCB variant that maintains a Cholesky-style representation of the inverse regularized design matrix as a sum of diagonal term and dynamically updated low-rank component. Using a projector-splitting method \cite{lubich2014projector}, our algorithm avoids explicit matrix inversion. Our method naturally supports vectorized and batch updates over multiple context vectors, which is crucial for efficient deployment in modern recommender systems. PSI-LinUCB achieves an average per-interaction update cost of $\mathcal{O}(dr)$ under the proper choice of batch size, with $r$ being the approximation rank. The inference complexity is $\mathcal{O}(dr)$ per action.
\item We empirically show that PSI-LinUCB is robust with respect to the choice of the approximation rank: increasing $r$ leads to a smooth and nearly linear improvement in recommendation quality (measured by hit rate), making the method easy to tune in practice.
\item We evaluate PSI-LinUCB on several large-scale datasets from recommender systems domain, where standard LinUCB fails to run due to scalability constraints. We ensure reduction in  memory consumption, and computational time compared to the exact Sherman–Morrison implementation of LinUCB \cite{Angioli2025Efficient}, while matching the hit rate of the vanilla LinUCB algorithm. We also show that PSI-LinUCB outperforms sketching-based baselines in terms of computational time, while achieving similar or better quality, both on large-scale datasets and in online synthetic environments. 
\end{itemize}

\section{Scalable LinUCB with Low-Rank Updates}
\label{sec:low_rank}
The main computational cost  of LinUCB lies in updating the parameter estimate $\hat{\theta}_a$ for each arm, which requires access to the inverse matrix $A_{t,a}^{-1}$. Existing approaches mitigate this issue either by applying rank-1 Sherman–Morrison updates \cite{Angioli2025Efficient} or by approximating the covariance matrix
\[
\textstyle
\sum_{s=1}^{t-1} {x}_{s,a} {x}_{s,a}^\top \in \rset^{d \times d}\eqsp,
\]
using low-rank sketching techniques such as Frequent Directions \cite{kuzborskij2019efficient} and CBSCFD \cite{chen2020efficient}. In contrast to these works, we employ an alternative approach, which relies on the dynamic representation of the \emph{inverse} matrix $A_{t,a}^{-1}$. Towards this aim, we use the Cholesky-style symmetric factorization $A_{t,a} = L_{t,a} L_{t,a}^\top$ and note that in this case
\begin{equation}
\label{eq:inverse_A_matrix_Cholesky}
A_{t,a}^{-1} = L_{t,a}^{-\top} L_{t,a}^{-1}\eqsp.
\end{equation}

The motivation for working with the inverse design matrix is that in LinUCB the design matrix $A_{t,a}$ is never used directly: both the parameter estimate
\[
\hat{\theta}_{t,a} = A_{t,a}^{-1} b_{t,a}
\]
and the exploration bonus
\[
\sqrt{x_{t,a}^\top A_{t,a}^{-1} x_{t,a}}
\]
involve only matrix--vector products with $A_{t,a}^{-1}$. Therefore, maintaining or approximating $A_{t,a}$ itself is algorithmically unnecessary.

The Cholesky-style factorization \eqref{eq:inverse_A_matrix_Cholesky} makes this observation particularly convenient. Indeed, it allows both quantities above to be computed via matrix-vector operations involving $L_{t,a}^{-1}$ only. This representation enables us to work directly with the inverse operator without explicitly forming or storing $A_{t,a}^{-1}$.

The representation \eqref{eq:inverse_A_matrix_Cholesky} is a key formula for our further analysis. Importantly, we \emph{do not} explicitly compute the inverse factors $L_{t,a}^{-1}$. Instead, they are updated dynamically from their previous values. Since the update rules are identical across arms, we omit the index $a$ in the remainder of this section and write $A_{t}$ and $L_t$ (respectively, $L_t^{-1}$) instead of $A_{t,a}$ and $L_{t,a}$.

\par 
In the next parts of this section, we present our methodology in different setups. First, in \Cref{sec:rank-1-upd} we consider the case of rank-1 updates and write the dynamics of updates of the inverse root matrix $L_{t}^{-1}$. Then in \Cref{sec:batch-updates} we generalize our expressions for the batch update case. Both representations rely on a recursive definition in the form of matrix decomposition $U_t V_t^\top$. In \Cref{sec:psi}, we show that this decomposition can be dynamically updated in a low-rank format using the Projector-Splitting Integrator (PSI) method of Lubich and Oseledets \cite{lubich2014projector}. Such updating mechanism prevents an uncontrollable growth of factor matrices $U_t$ and $V_t$ over time and enables ``on-the-fly'' adaptation to the stream of contextual data. This approach has been previously considered in the context of modifying the PureSVD model in recommender systems \cite{Olaleke2021Dynamic}. 

\subsection{Rank-1 updates}
\label{sec:rank-1-upd}
We first consider the rank-1 update setting, where context vectors arrive sequentially and model parameters are updated after each interaction. In this case, the design matrix update at time $t+1$ takes the form
\begin{equation}
\label{eq:rank1_A_t}
A_{t+1} = A_t + x_{t+1}x_{t+1}^{\top}.
\end{equation}
While inverse updates can be obtained via the Sherman–Morrison formula \cite{Angioli2025Efficient}, we instead derive an alternative representation that enables efficient low-rank updates within our proposed framework.

We begin by expressing the rank-1 update in a Cholesky-style factorized form. Using $A_t = L_t L_t^\top$, the update at time $t+1$ can be written as 
\begin{equation}
A_{t+1} = L_tL_t^{\top} + x_{t+1}x_{t+1}^{\top} = L_t(I + L_t^{-1}x_{t+1}x_{t+1}^{\top}L_t^{-\top})L_t^{\top}.
\end{equation}
Hence, we can rewrite the design matrix as
%updated Cholesky-style symmetric factor writes as
\[
A_{t+1} =L_{t+1}L_{t+1}^{\top}\eqsp,
\]
where we set
\begin{equation}
\label{eq:def_l_t_plus_1_rec}
L_{t+1} = L_t(I+ \alpha_{t+1} \tilde{x}_{t+1}\tilde{x}_{t+1}^{\top})\eqsp, \eqsp \tilde{x}_{t+1} = L_t^{-1}x_{t+1}\eqsp,
\end{equation}
and parameter $\alpha_{t+1} \in \rset$ such that 
\[
1 + \alpha_{t+1}\,\|\tilde{x}_{t+1}\|^2
\;=\;
\sqrt{\,1 + \|\tilde{x}_{t+1}\|^2\,}\eqsp.
\]
The correctness of the specified recursive definition of $L_{t+1}$ in \eqref{eq:def_l_t_plus_1_rec} can be verified by direct substitution. The following theorem for the rank-1 updates setting will be useful for illustrating the essence of our approach. 

\begin{theorem}
\label{prop: rank-one update}
Let $\varepsilon > 0$, $L_0^{-1} = \varepsilon^{-1/2} I \in \rset^{d \times d}$, and $U_0$, $V_0$ be empty matrices. Given a sequence of context vectors $\{x_t\}_{t \in \nset}$  set
\[
\beta_{t+1} \;=\; \frac{\alpha_{t+1}}{1 + \alpha_{t+1}\,\|\tilde{x}_{t+1}\|^2}.
\]
Then the inverse root \(L_{t+1}^{-1}\) can be expressed as
\[
L^{-1}_{t+1}
\;=\;
\bigl(I - U_{t+1}\,V_{t+1}^\top\bigr)\,L_0^{-1},
\]
where the matrices $U_{t+1} \in \rset^{d \times (t+1)}$ and $V_{t+1} \in \rset^{d \times (t+1)}$ are recursively updated with column-wise concatenation
\[
U_{t+1}
=
\begin{bmatrix}
U_t & \beta_{t+1}\,\tilde{x}_{t+1}
\end{bmatrix},
V_{t+1}
=
\begin{bmatrix}
V_t & (I - V_t\,U_t^\top)\,\tilde{x}_{t+1}
\end{bmatrix}.
\]
\end{theorem}
\begin{proof}
    The proof is provided in the Appendix, \Cref{sec:proof_rank-one update}
\end{proof}

\subsection{Batch updates}
\label{sec:batch-updates}

Now we provide a generalization of \Cref{prop: rank-one update} to the case of batch updates. In this setting parameter is updated after some number of interactions is accumulated in the system. Formally, we write $X_t \in \rset^{d \times B}$ for the batch of $B \in \nset$ concatenated contexts collected during the $t$-th interaction round with the arm $a \in \mathcal{A}$. Note that, generally speaking, $B$ depends on $a$ and $t$, but we prefer to write $B$ instead of $B_{t,a}$ for notation simplicity.

The update rule for the regularized design matrix can then be expressed as
\[
A_{t+1} = A_t + X_{t+1}X_{t+1}^\top = L_{t+1}L_{t+1}^\top,
\]
where the Cholesky-like factor $L_{t+1}$ is obtained recursively: 
\begin{equation}
\label{eq:batch_update_L_t}
L_{t+1} = L_t (I + Q_{t+1} (M_{t+1} - I) Q_{t+1}^\top),
\end{equation}
with $Q_{t+1} \in \mathbb{R}^{d \times B}$ being a matrix with orthonormal columns and 
$M_{t+1} \in \mathbb{R}^{B \times B}$ obtained using the fast symmetric factorization approach \cite{Ambikasaran} described in \Cref{app:proofs} (see \Cref{theorem: sym_fact}). 

\begin{theorem}
\label{prop: batch update_rec}
 Let $\varepsilon > 0$, $L_0 = \varepsilon^{-1/2}I \in \rset^{d\times d}$ and $U_0, V_0$ be empty matrices, and $U_t, V_t \in \rset^{d\times d_t}$. Then the factor $L_{t+1}$ defined in \eqref{eq:batch_update_L_t} can be updated by formula 
 \begin{equation}
 \label{eq:L_t_definition}
 L_{t+1}^{-1} = (I - U_{t+1}V_{t+1}^\top)L_0^{-1}.
 \end{equation}
Matrices $U_{t+1}, V_{t+1} \in \rset^{d\times(d_t + B)}$ are updated with column-wise concatenation
\begin{equation}
\label{eq:U_t_v_t_def_block}
\begin{split}
U_{t+1} &= 
\begin{bmatrix}
U_t \quad Q_{t+1} (M_{t+1} - I)M_{t+1}^{-1}
\end{bmatrix}\eqsp, \\
V_{t+1} &= 
\begin{bmatrix}
V_t \quad (I - V_tU_t^\top)Q_{t+1}
\end{bmatrix}\eqsp.
\end{split}
\end{equation}
\end{theorem}

Proof of \Cref{prop: batch update_rec} is provided in \Cref{sec:proof_batch update_rec}. The shape $d_t$ of the matrices $U_{t}$ and $V_{t}$ can be inferred from the representations \eqref{eq:L_t_definition} - \eqref{eq:U_t_v_t_def_block}.

\subsection{Low-rank correction for $A_t^{-1}$}
Note that the matrices $U_t$ and $V_t$ in \eqref{eq:U_t_v_t_def_block} expand in the number of columns with time $t$, leading to excessive memory demands and computational overhead. We now motivate why the correction term for 
$\varepsilon^{-1} I - A_t^{-1}$ can be well approximated by a low-rank matrix. Consider the empirical covariance matrix
$\sum_{s=1}^t x_s x_s^\top,$ and let its SVD be given by
\[
\textstyle 
\sum_{s=1}^t x_s x_s^\top = E_x \Sigma_x E_x^\top.
\]
Then the regularized design matrix and its inverse admit the representations
\[
A_t = E_x (\varepsilon I + \Sigma_x) E_x^\top,
\qquad
A_t^{-1} = E_x (\varepsilon I + \Sigma_x)^{-1} E_x^\top.
\]
Then the correction term $R_t := \varepsilon^{-1} I - A_t^{-1}$ writes as
\[
R_t
= E_x \Bigl(\varepsilon^{-1} I - (\varepsilon I + \Sigma_x)^{-1}\Bigr) E_x^\top.
\]
The diagonal entries of the matrix inside the parentheses are given by
\[
\textstyle 
\varepsilon^{-1} - \frac{1}{\varepsilon + \sigma_i}
= \frac{\sigma_i}{\varepsilon(\varepsilon + \sigma_i)}
\le \frac{\sigma_i}{\varepsilon^2}.
\]
Therefore, eigencomponents corresponding to small eigenvalues $\sigma_i$ contribute negligibly to $R_t$. Moreover, approximating the empirical covariance $\sum_{s=1}^t x_s x_s^\top$ with rank $r$ before inversion $A_t$ or approximating only the correction $R_t$ with rank $r$ in decomposition for $A_t^{-1}$ leads to the same approximation of $A_t^{-1}$.
This observation supports modeling $R_t$ using a low-rank representation when the empirical covariance matrix $\sum_{s=1}^t x_s x_s^\top$ has low effective rank, a property commonly employed in the literature on sketching methods \cite{kuzborskij2019efficient,chen2020efficient}. Below we describe how to maintain a low-rank approximation of $R_t$ by controlling the ranks of $U_t$ and $V_t$ using the Projector-Splitting Integrator \cite{lubich2014projector}. The motivation for this particular approximation method is discussed in the next section.

\section{Projector-Splitting Integrator}
\label{sec:psi}
% Now we aim to improve the LinUCB algorithm based on the representation \eqref{eq:L_t_definition}-\eqref{eq:U_t_v_t_def_block}. As already mentioned, the matrices $U_t$ and $V_t$ here grow with time $t$, which yields both computational and memory issues. At the same time, the key element of the representation \eqref{eq:L_t_definition} is the time-varying factor $D_t = U_t V_t^{\top} \in \rset^{d \times d}$. In this case, we essentially observe the matrix dynamics $\{D_{t}\}_{t \in \nset}$. Moreover, the matrices $D_t$ might be well-approximated by low-rank factors. One of the options for dynamic approximation of $D_t$ comes from a singular value decomposition
We now aim to improve LinUCB using the representation in \eqref{eq:L_t_definition}–\eqref{eq:U_t_v_t_def_block}. In this form, the matrices $U_t$ and $V_t$ grow with time. At the same time, the key element of the representation \eqref{eq:L_t_definition} is the time-dependent matrix $D_t = U_t V_t^\top \in \rset^{d \times d}$, whose evolution defines the dynamics of the representation. 
A natural approach is to approximate $D_t$ via a  rank-$r$ truncated SVD with,
\[
\textstyle
D_t = \bar{U}_t \Sigma_{t} \bar{V}_t^{\transpose}\eqsp.
\]
 At the same time, the orthogonal factors $(\bar{U}_{t},\bar{V}_t)$ and $(\bar{U}_{t+1},\bar{V}_{t+1})$, corresponding to the matrices $D_t$ and $D_{t+1}$, are not guaranteed to be close, which might yield additional computational instability. Instead, we propose to rely on the projector-splitting integrator (PSI) approach of \cite{lubich2014projector}, which constructs a dynamical low-rank approximation by solving
\[
\textstyle 
\norm{\dot{D}_t - \dot{\tilde{D}}_t} \to \min_{\tilde{D}_t: \, \operatorname{rank}{\tilde{D}}_t = r}\,,
\]
where $\dot{D}_t = \frac{\rmd}{\rmd t}D_t$. The PSI method enables efficient iterative updates without explicitly forming $\tilde{D}_t$. The approximation is maintained in factorized form,
\[
\textstyle
\tilde{D}_t = \tilde{U}_{t} S_t \tilde{V}_t^{\top}\eqsp,
\]
where the factors $\tilde{U}_{t},\tilde{V}_{t} \in \rset^{d \times r}$ have orthonormal columns, and $S_t \in \rset^{r \times r}$ is invertible. Then we obtain factors $(\tilde{U}_{t+1}, S_{t+1}, \tilde{V}_{t+1}^{\top})$, such that 
\[
\tilde{D}_{t+1} = \tilde{U}_{t+1} S_{t+1} \tilde{V}_{t+1}^{\top}\eqsp.
\]
We provide implementation details in \Cref{alg: PSI}. 
\par 
% Now we explain the main steps of integrating \Cref{alg: PSI} into the LinUCB algorithm with batch updates.
% At each batch update step, two matrices, $U_t$ and $V_t$, are maintained to represent the low-rank structure.
% Initially, these matrices are incrementally expanded by \eqref{eq:rank1_UV} or \eqref{eq:U_t_v_t_def_block} as new data arrive, until their number of columns exceed a predefined rank threshold $r$. 
% Once this occurs at some time step $t_0$, we compute the singular value decomposition (SVD) with fixed rank $r$ of the current product $U_{t_0} V_{t_0}^\top = \tilde U_{t_0}S_{t_0}\tilde V_{t_0}^\top$ to obtain the new factors $ U_{t_0} = \tilde U_{t_0}S_{t_0}$, and $V_{t_0} = \tilde V_{t_0}$ for initialization \Cref{alg: PSI}. In subsequent batch updates, the matrices  $U_t$ and $V_t$ are efficiently updated using the \Cref{alg: PSI}, which maintains the r-rank approximation. 
% The corresponding increments $\Delta D_{t+1}$ defined by the algorithm's update mechanism $U_{t+1} V_{t+1}^\top = U_t V_t^\top + \Delta D_{t+1}$ are directly obtained from \eqref{eq:rank1_A_t} for the rank-1 update case:
We integrate this procedure into LinUCB with batch updates as follows. During training, the factors $U_t$ and $V_t$ are incrementally expanded using \eqref{eq:U_t_v_t_def_block} or \eqref{eq:rank1_UV} until their number of columns exceed a predefined threshold $r$. At this time $t_0$, we compute a rank-$r$ SVD of $U_{t_0}V_{t_0}^\top$ to initialize the PSI factors. Subsequent updates apply \Cref{alg: PSI} to maintain a fixed-rank approximation. The update increments $\Delta D_{t+1}$ are obtained directly from the LinUCB update rules. For rank-one updates,
\begin{equation}
    \Delta D_{t+1} = \beta_{t+1} \tilde{x}_{t+1} \tilde{x}_{t+1}^\top (I - U_t V_t^\top);
\end{equation}
and from \eqref{eq:U_t_v_t_def_block} for the batch update case:
\begin{equation}
    \Delta D_{t+1} = Q_{t+1}(M_{t+1}-I)M_{t+1}^{-1} Q_{t+1}^\top (I - U_t V_t^\top).
\end{equation}
The complete integration of PSI into LinUCB is summarized in \Cref{alg:linucb-psi}. We highlight that we do not need to form the matrices $L_{t}^{-1}$ given in \eqref{eq:L_t_definition} explicitly. 

\textbf{Complexity analysis.} We analyze the computational complexity of the proposed PSI-LinUCB algorithm during training. The complexity of the PSI update
\[
U_{t+1}, V_{t+1} =\textit{PSI}(U_{t},V_{t},\;\Delta D_{t+1}) 
\]
does not depend on the number of interactions $B = B_{t,a}$ (that is, interactions with arm $a$ inside batch number $t$) and scales as $\mathcal{O}(dr^2)$. This factor comes from the QR decomposition applied to $\rset^{d \times r}$ matrices. When the rank exceeds the threshold, an SVD is computed once per arm (line~11 in \Cref{alg: update_hand}) with complexity $\mathcal{O}(d(r+B)^2 + (r+B)^3)$, obtained via QR decompositions of factors of size at most $\rset^{d \times (r+B)}$. Additionally, line~5 in \Cref{alg: update_hand} incurs a cost of $\mathcal{O}(dB^2 + B^3)$ due to QR and Cholesky decompositions. Thus, the overall complexity of handling blocks in the \textit{Update\_arm} algorithm (\Cref{alg: update_hand}) is 
\[
\mathcal{O}(d(r^2+B^2) + r^3+B^3)\eqsp,
\]
as soon as this arm has accumulated at least $r$ interactions.

\textit{Inference stage.} To compute the bonus term in \eqref{eq:arm_choice_lin_ucb} we use the decomposition defined earlier:
\begin{equation}
\label{eq:bonus_PSI}
\begin{aligned}
    \sqrt{x_{t,a}^\top A_t^{-1} x_{t,a}} 
    &= \sqrt{x_{t,a}^\top (L_t L_t^\top)^{-1} x_{t,a}} \\
    &= \| (I - U_t V_t^\top)L_0^{-1} x_{t,a} \|
\end{aligned}
\end{equation}
Then at time $t$, the arm is selected by formula:
\begin{equation}
\begin{split}
a_t = \arg\max_{a \in \mathcal{A}_t} 
\left( 
\theta_a^\top x_{t,a} + \alpha \| (I - U_t V_t^\top)L_0^{-1} x_{t,a} \|
\right)\eqsp.
\end{split}
\end{equation}
The computational complexity of this operation is $\mathcal{O}(dr)$ per arm.
\begin{table}[t!]
\centering
\footnotesize
\caption{Computational complexity comparison across algorithms}
\label{tab:complexity}
\setlength{\tabcolsep}{8pt}
\renewcommand{\arraystretch}{1.2}

\begin{tabular}{lcc}
\hline
\textbf{Algorithm} & \textbf{Time cost per round} & \textbf{Space} \\
\hline
LinUCB    & $O(d^2)$              & $O(d^2)$ \\
CBRAP                & $O(dm + m^3)$         & $O(dm)$  \\
CBSCFD                & $O(dm)$         & $O(dm)$  \\
DBSLinUCB           & $O(dl_{B_t})$               & $O(dl_{B_t})$  \\

PSI-LinUCB           & $O(dr)$               & $O(dr)$  \\

\hline
\end{tabular}
\end{table}

\begin{remark}
\label{rem}
The primary computational cost of our algorithm arises during the training phase, which requires on average $\mathcal{O}(d(r^2/B + B) + r^3/B + B^2)$ operations per interaction. By choosing $B \propto r$ and assuming $r \ll d$, this reduces to average complexity of $\mathcal{O}(dr)$ operations per iteration. To our knowledge, the $\mathcal{O}(dr)$ average complexity is the best computational cost achieved by LinUCB-based algorithms. This is comparable with the most sample-efficient implementations of existing algorithms \cite{chen2020efficient, wen2024matrix}, summarized in Table~\ref{tab:complexity}.
\end{remark}

%\textbf{Theoretical properties.} We do not provide a formal regret bound for PSI-LinUCB. Existing regret analyses for LinUCB and its variants crucially rely on the monotonicity of the empirical covariance matrix or its estimates. In our setting, we directly approximate the inverse regularized design matrix, for which such monotonicity does not hold. One could hope to rely on norm-based bounds for the approximation error of the inverse matrix. However, even such guarantees are currently unavailable, since the approximation error of the projector-splitting integrator used in PSI-LinUCB is not theoretically characterized. Moreover, to the best of our knowledge, there are no regret guarantees for linear bandit algorithms based on approximating the inverse design matrix. Nevertheless, for a sufficiently large approximation rank, our method exactly recovers the inverse matrix and reduces to standard LinUCB, suggesting that the regret should not be worse for our method.

\textbf{Theoretical properties.} Assume that the matrix $U_{t}V_{t}^{\top}$ from the exact LinUCB update representation \Cref{prop: batch update_rec} has rank not exceeding $r$ for any $t \in \{0,\ldots,T\}$. Then theoretical guarantees for PSI integrator applied with the same rank $r$ \cite{lubich2014projector}[Theorem~4.1] ensures exact integration, that is, \Cref{alg: PSI} maintains the exact $L_t^{-1}$ and and hence the exact inverse $A_t^{-1} = L_t^{-\top}L_t^{-1}$. Consequently, the estimator $\hat\theta_t=A_t^{-1}b_t$ and the confidence
bonus in \eqref{eq:bonus_PSI} exactly matches those of the standard LinUCB algorithm \eqref{eq:arm_choice_lin_ucb}. In this setting, PSI-LinUCB admits the standard regret scaling of LinUCB under linear bandit assumptions of order $\tilde O(d\sqrt{T})$ with high probability after $T$ iterates, see \cite{abbasi2011improved}. 

Investigating the setting when the exact LinUCB algorithm matrix $U_{t}V_{t}^{\top}$ has larger rank is an important direction for future work. Existing regret analysis of LinUCB and its sketching variants \cite{kuzborskij2019efficient} rely on the monotonicity properties of the estimates of feature covariance matrix. In our setting, we directly approximate the inverse regularized design matrix, and such monotonicity does not hold, making standard techniques unavailable.

\begin{algorithm}
\caption{\textit{PSI-LinUCB (LinUCB training with PSI)}\, \, %$O(B^3 + B^2d+drB+ dr^2)$ 
}
\label{alg:linucb-psi}
\begin{algorithmic}[1]
\REQUIRE train\_data = $[batch_0, \dots, batch_{n-1}]$, batch\_size, $U\_0,\, V\_0$
\STATE $L_0^{-1} = \varepsilon^{-1/2}I$ 
\FOR{t in \{0,\ldots,n-1\}}
  \STATE $X_{t,a} = []$, $R_{t,a} = []$
  \FORALL{$(u,a,r)$ in $batch_{t}$}
    \STATE $X_{t,a}.append\,[x_{t,a}]$, $R_{t,a}.append\,[r]$   
  \ENDFOR
  \FOR{each arm $a$ in $batch_t$}
    \STATE $U_{t+1,a},V_{t+1,a},\theta_{t+1,a} \gets $ \textit{Update\_arm($a$)}
  \ENDFOR
\ENDFOR
 \STATE \textbf{return} $\theta_{t,a}$, $U_{t,a}, V_{t,a}$ for each arm $a \in \mathcal{A}$. 
\end{algorithmic}
\end{algorithm}

\begin{algorithm}
\caption{\textit{Update\_arm}($a$)\,}
\label{alg: update_hand}
\begin{algorithmic}[1]
    \REQUIRE $b_{t,a}$, $U_{t,a}$, $V_{t,a}$, $X_{t,a}$, $R_{t,a}$
    \STATE Set $b_t = b_{t,a}$, $U_{t} = U_{t,a}$, $V_{t} = V_{t,a}$, $X_t = X_{t,a}$, $R_t = R_{t,a}$, \\
    \COMMENT{Omitting index $a$ for simplicity}
    \STATE $b_{t+1} = b_t + X_tR_t$
    \STATE $L_{t}^{-1}=(I - U_{t}\,V_{t}^\top)L_0^{-1}$
    \linecomment{not form $L_{t}^{-1}$ explicitly}
    \STATE $\bar X_{t+1} =L_t^{-1} X_t$
    \STATE $C_{t+1}, Q_{t+1} \gets$ \textit{Calculate\_C\_and\_Q}$(\bar X_{t+1})$
    \IF{$U_t.\text{shape}[1] < r$}
        \STATE $U_{t+1} \;\gets\; [\,U_t,\;Q_{t+1} C_{t+1}]$
        \STATE $V_{t+1} \;\gets\; [\,V_t,\, (I -V_tU_t^\top)\; Q_{t+1}\,]$
    \ELSE
        \IF{first time $U_t.\text{shape}[1] \geq  r$}
            \STATE $\tilde U_{t+1}S_{t+1}\tilde V_{t+1}^\top = SVD(U_tV_t^T)$
            \STATE $U_{t+1} \;\gets\; \tilde U_{t+1}S_{t+1}$
            \STATE $V_{t+1} \;\gets\; \tilde V_{t+1}$
        \ELSE
            \STATE $\Delta D_{t+1} = 
        \,Q_{t+1}\, C_{t+1} \, Q_{t+1}^\top\,(I - U_t V_t^\top),$ \\ \COMMENT{We do not form $\Delta D_{t+1}$ explicitly}
            \STATE $ U_{t+1}, V_{t+1} =\textit{PSI}(U_{t},V_{t},\;\Delta D_{t+1})$
        \ENDIF
    \ENDIF
    % \STATE $(U_{t+1}, V_{t+1})\;\gets\;\textit{PSI}(U_{t+1}, V_{t+1},\;\Delta D_{t+1})$
    \STATE $L_{t+1}^{-1}=(I - U_{t+1}\,V_{t+1}^\top)L_0^{-1}$
    \linecomment{We do not form $L_{t+1}^{-1}$ explicitly}
    \STATE $\theta_{t+1} = L_{t+1}^{-\top}\,L_{t+1}^{-1}\;b_{t+1}$
    \STATE \textbf{return} $U_{t+1}, V_{t+1}$, $\theta_{t+1}$
\end{algorithmic}  
\end{algorithm}

\begin{algorithm}
\caption{\textit{PSI(Projector-Splitting Integrator)}\, \, %$O(dr^2)$
}
\label{alg: PSI}
\begin{algorithmic}[1]
 \REQUIRE $U_t,\,V_t, \Delta D_{t+1}$, $U_t,V_t \in \rset^{d \times r}$, $\Delta D_{t+1} \in \rset^{d \times d}$
  \STATE $K_1 \;\gets\; U_t\;+\;\Delta D_{t+1}\,V_t$
  \STATE $(\tilde U_1,\;\widetilde{S}_1)\;\gets\;\mathrm{QR}(K_1)$ 
  \STATE $\widetilde{S}_0 \;\gets\; \widetilde{S}_1 \;-\; \tilde U_1^\top\,\Delta D_{t+1}\,V_t$
  \STATE $L_1 \;\gets\; V_t\,\widetilde{S}_0^\top \;+\;\Delta D_{t+1}^\top\,\tilde U_1$
  \STATE $(\tilde V_1,\;S_1^\top)\;\gets\;\mathrm{QR}(L_1)$ 
\ENSURE $\tilde U_1S_1,\,\tilde V_1$ 
\end{algorithmic}
\end{algorithm}

\begin{algorithm} 
\caption{\textit{Calculate\_C\_and\_Q}  \, \, %$O(B^3 + B^2d)$
}
\begin{algorithmic}[1]
    \REQUIRE $\bar X_{t+1} \in \rset^{d \times B}$
    \STATE $Q_{t+1}, R_{t+1} \;\gets\;\mathrm{QR}(\bar X_{t+1})$   
    \STATE $T_{t+1} =I + R_{t+1}\,R_{t+1}^\top$ \linecomment{$T_{t+1} \in \rset^{B \times B}$}
    \STATE $M_{t+1} \gets \mathrm{Cholesky}(T_{t+1})$  \linecomment{Find $M_{t+1} \in \rset^{B\times B}$ such that $T_{t+1} = M_{t+1}\,M_{t+1}^T$}   
    %\STATE $Y =M - I$ \linecomment{$Y \in \rset^{B \times B}$}
    \STATE $C_{t+1} = (M_{t+1} - I)M_{t+1}^{-1}$
    \ENSURE  $C_{t+1}, Q_{t+1}$
\end{algorithmic}
\end{algorithm}

\section{Experimental Setup}
\label{sec:exp}

We evaluate the proposed PSI-LinUCB against five baselines covering exact updates, sketching and random-projection approaches:
\begin{itemize}[noitemsep,topsep=0pt,leftmargin=0em]
    \item LinUCB~\cite{li2010contextual}: the standard LinUCB algorithm with batched updates;
    \item LinUCB Classic ~\cite{Angioli2025Efficient}: LinUCB with exact rank-1 inverse updates via the Sherman--Morrison formula;
    \item CBSCFD ~\cite{chen2020efficient}:  a sketching-based method that approximates the feature covariance (design) matrix.
    \item CBRAP~\cite{yu2017cbrap}: method based on random projections applied to feature covariance (design) matrix.  
    \item DBSLinUCB~\cite{wen2024matrix}: an adaptive sketching method that dynamically adjusts the sketch size over time.
\end{itemize}
We provide more details about existing methods in the Appendix, \Cref{sec:related-work}.

\begin{table}[t!]
\centering
\small
\setlength{\tabcolsep}{1pt} 
\caption{Information about the datasets used for validation}
\label{tab:dataset_statistics}
\begin{tabular}{l|rrrr}
\toprule
\textbf{Dataset} & \textbf{\#Users} & \textbf{\#Items} & \textbf{\#Interactions} & \textbf{Density} \\
\midrule
Magazine Subscriptions & 60,100   & 3,400    &   71,500   & 0.35\% \\
Health \& Personal Care & 461,700  & 60,300   &  494,100   & 0.02\% \\
All Beauty & 632,000  & 112,600  &  701,500   & 0.01\% \\
MovieLens 1M & 6,040 & 3,706 & 1,000,209 & 4.18\% \\
\bottomrule
\end{tabular}
\end{table}

\subsection{Training Protocol}
\label{exp:dev}
In order to run our algorithm on the datasets, we convert them into the bandit style environment following the pipeline below: 
\begin{itemize}
[noitemsep,topsep=0pt,leftmargin=0em]
\item \textit{Warm-up phase}: Initial training on 80\% of historical data
\item \textit{Online training phase}: Sequential learning on the remaining 20\% of data, divided into equal-sized temporal intervals simulating monthly updates. The model is incrementally trained on each month's data.
\item \textit{Evaluation}: Model evaluated after each monthly training update to assess recommendation quality over time.
\end{itemize}

In the sections below, we consider the datasets described in \Cref{tab:dataset_statistics}. The LinUCB algorithm is implemented with batch processing to ensure scalability with respect to high-dimensional contextual features. User-item features are extracted via SVD decomposition of the user-item interaction matrix, with rank $r'$ selected through spectral analysis during warm-up. To obtain high-dimensional contexts, the context vectors are constructed as outer products of user and item embeddings. All hyperparameters are tuned via cross-validation for each dataset and model, see \Cref{appendix:tuning} in Appendix.
\section{Scalability results}

\textbf{Increasing context size.} To ensure the scalability of PSI-LinUCB, we vary the dimensions of context vectors $x_{t,a}$. We infer user-item features from SVD with varying rank $r'$ on the subset of the Amazon Health dataset, and measure the computational time and memory usage. As shown in \Cref{fig:context_mag}, LinUCB and LinUCB Classic fail to scale beyond moderate context sizes due to memory and time requirements for storing and inverting the full matrix $A_{t,a}$. In contrast, PSI-LinUCB operates efficiently even for large $d$ by maintaining only low-rank factors with $\mathcal{O}(dr)$ complexity. As shown in \Cref{tab:hit_rate_d}, PSI-LinUCB achieves identical quality to LinUCB where both are feasible, while continuing to operate at larger context sizes. Warm-up phase results are provided in Appendix~\ref{app:experiments}, see \Cref{fig:context_mag_train}.

\begin{figure}[t!]
    \centering
    \includegraphics[width=1\columnwidth]{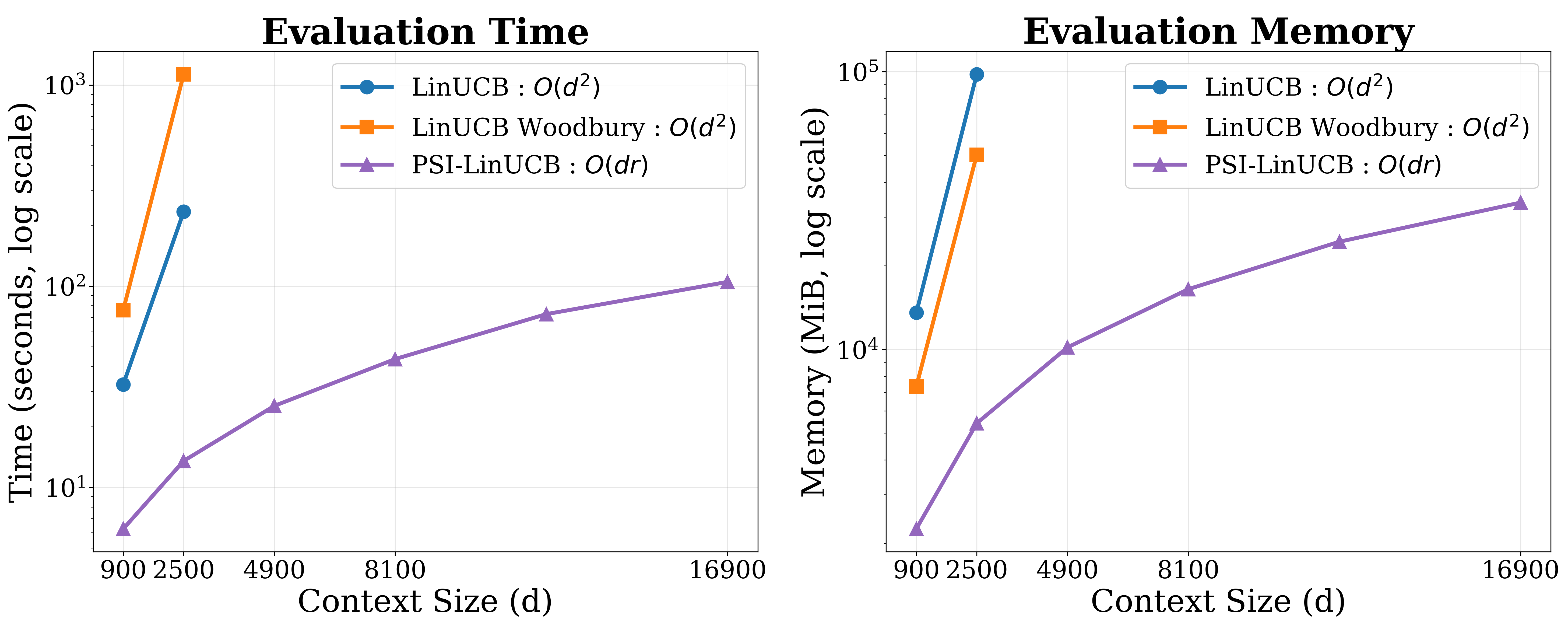}
    \caption{Performance comparison across different context sizes on Amazon Health dataset.}
    \label{fig:context_mag}
\end{figure}

\begin{table}[t!]
\centering
\small
\caption{Hit Rate comparison across context dimensions $d$}
\label{tab:hit_rate_d}
\begin{tabular}{c|ccc}
\toprule
$d$ & LinUCB & LinUCB Woodbury & PSI-LinUCB \\
\midrule
900  & 0.026 & 0.026 & 0.026 \\
2500 & 0.034 & 0.034 & 0.034 \\
4900 & - & - & 0.044 \\
16900 & - & - & 0.050 \\
\bottomrule
\end{tabular}
\end{table}

\textbf{Increasing number of arms.} We fix the optimal context dimensions and PSI rank for each dataset (Health, Beauty, and Magazine Subscriptions), then gradually increase the number of arms from a small subset to the full action space, selecting items with the highest observation frequency at each stage. As shown in \Cref{fig:arms_beauty}, exact implementations of LinUCB become grossly memory-inefficient as number of arms increases, while PSI-LinUCB remains computationally tractable. Additional results on time and memory consumption during the warm-up phase on the Beauty dataset (Figure~\ref{fig:arms_beauty_train}), as well as further scalability experiments are reported in the Appendix, ~\Cref{app:experiments}.

\begin{figure}[t!]
\centering
\includegraphics[width=1\columnwidth]{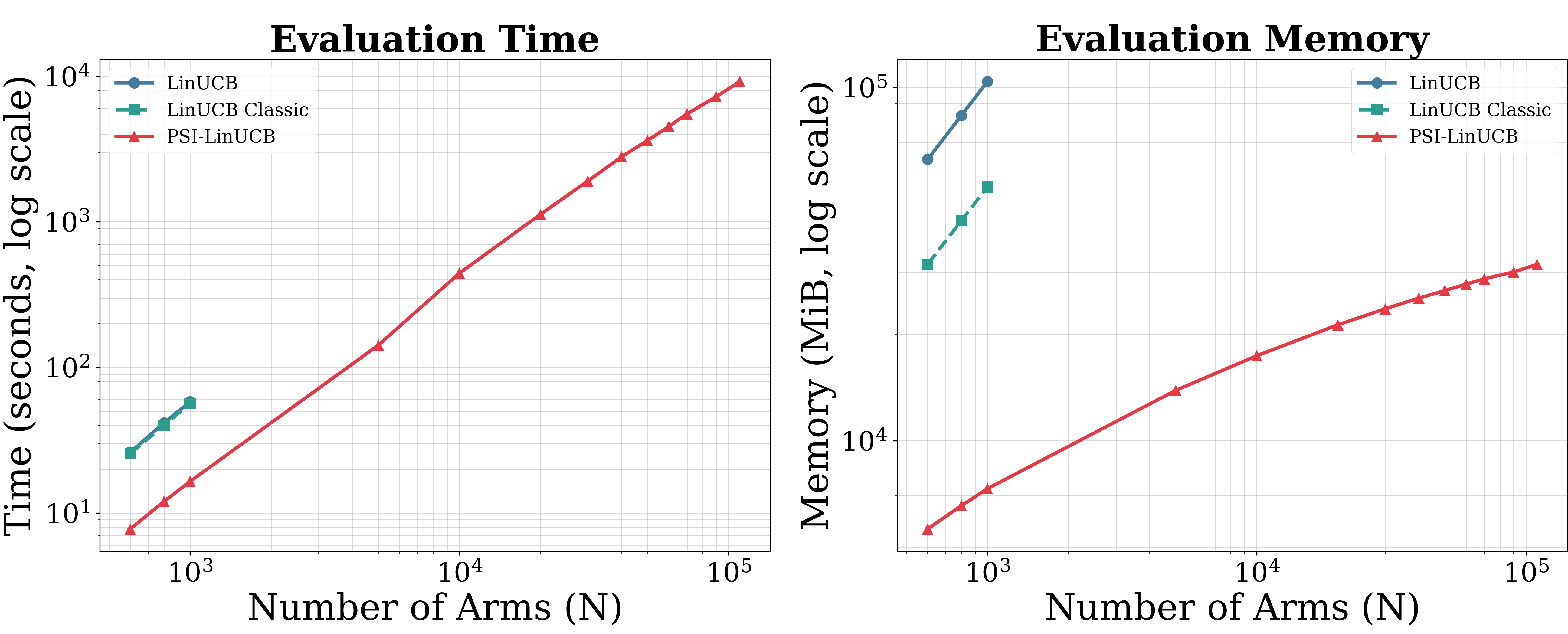}
\caption{Algorithm scaling with number of arms on Beauty dataset.}
\label{fig:arms_beauty}
\end{figure}

\textbf{LinUCB quality.} We analyze how the rank parameter affects recommendation quality by varying the PSI rank while keeping context dimensions and number of arms fixed. At each configuration, we measure hit rate of our algorithm, popular and random baselines relative to the one of LinUCB. \Cref{fig:quality_comparison_amazon} shows that PSI-LinUCB achieves recommendation quality comparable to classical LinUCB starting from relatively low ranks, without requiring a full-rank matrix. Additional results are provided in Appendix, see \Cref{fig:quality_comparison_other}.

\begin{figure}[t!]
  \centering
  \begin{subfigure}[b]{0.49\textwidth}
    \includegraphics[width=\textwidth]{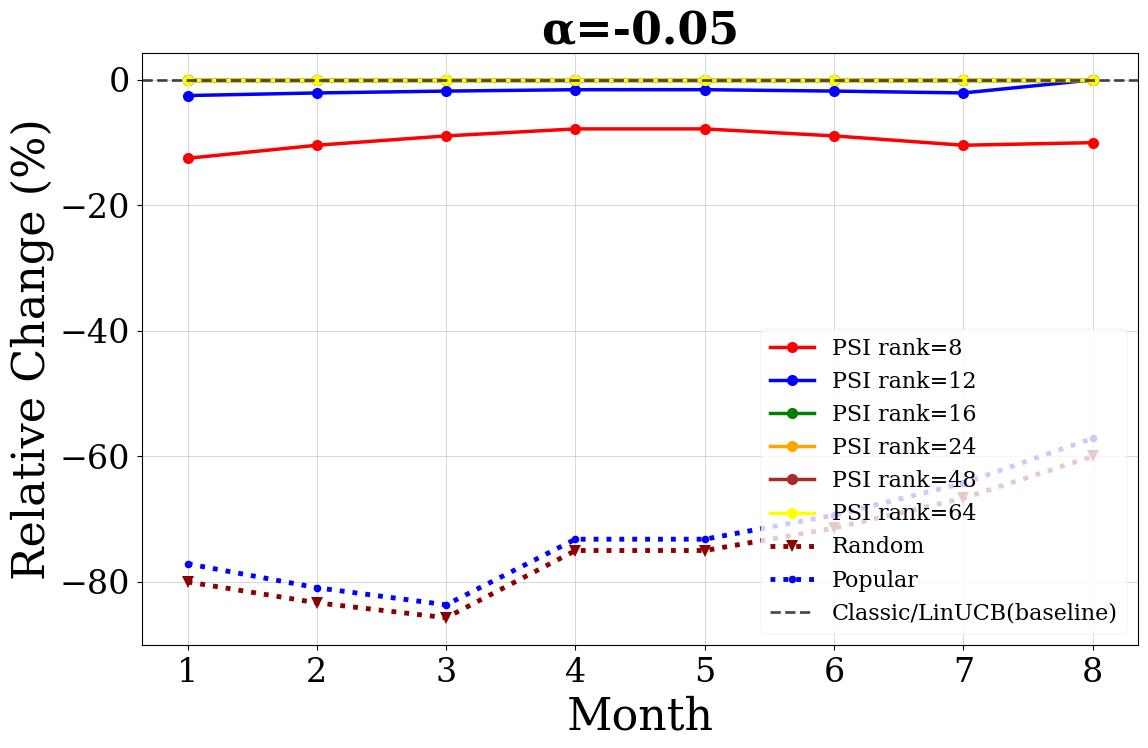}
    \caption{Amazon Health}
    \label{fig:quality_health}
  \end{subfigure}
  \hfill
  \begin{subfigure}[b]{0.49\textwidth}
    \includegraphics[width=\textwidth]{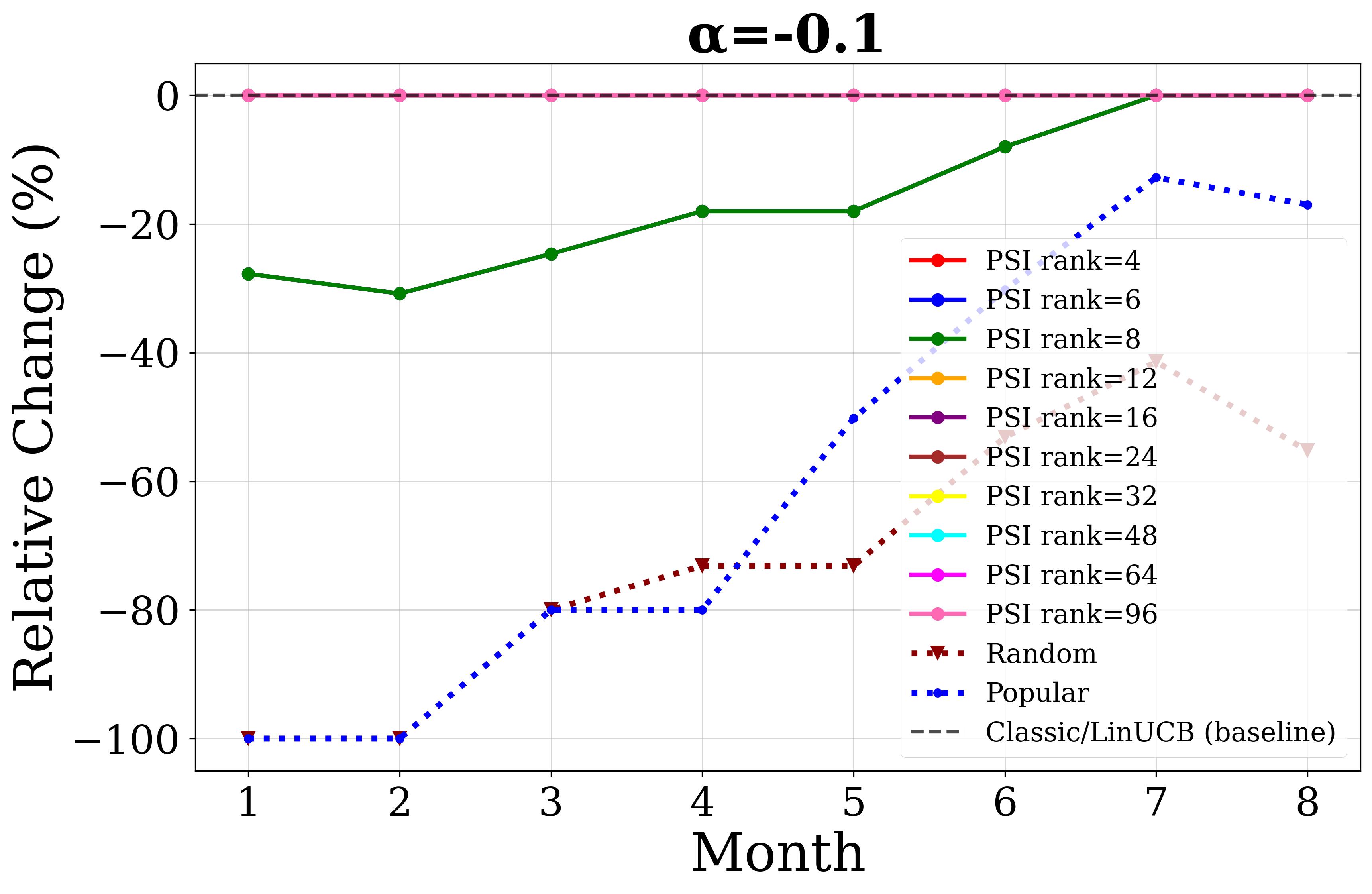}
    \caption{Amazon All Beauty}
    \label{fig:quality_beauty}
  \end{subfigure}
  \caption{Quality comparison by months: PSI-LinUCB vs LinUCB on Amazon datasets.}
  \label{fig:quality_comparison_amazon}
\end{figure}

\section{Comparison with scalable baselines}
\subsection{Online setting}
\label{sec:exp6}
Following the experimental setup of~\cite{kuzborskij2019efficient} and~\cite{chen2020efficient}, we evaluate our method on an online classification problem turned into contextual bandit setting. To maintain compatibility with the mentioned works, in this subsection we switch to the version of LinUCB with shared parametrization, see e.g. \cite{chen2020efficient} with a single parameter vector and design matrix across all arms. 

Given a dataset with $K$ classes, we mark of them as a target. In each round, the environment randomly draws one sample from each class, forming a context set. The learner selects one sample and receives a reward of $1$ if the selected sample belongs to the target cluster, and $0$ otherwise. Note that in this binary reward setting, cumulative number of mistakes corresponds to the cumulative regret. We conduct experiments on two benchmark datasets: MNIST~\cite{LeCun1998Gradient} and CIFAR-10~\cite{Krizhevsky2009Learning}. For fair comparison we tuned the optimal rank/sketch size for all algorithms via cross-validation by minimizing mistakes over $T$ rounds. The dataset statistics are summarized in the Appendix, \Cref{tab:cifar_mnist}, the detailed description of the experimental setup and hyperparameter selection is provided in \Cref{sec:exp5}. Since this experiment focuses on online learning, we use the rank-1 version of PSI-LinUCB (see \Cref{alg:linucb-psi-rank1} in Appendix), which is consistent with the update scheme employed by CBSCFD~\cite{chen2020efficient}, CBRAP~\cite{yu2017cbrap}, and DBSL~\cite{wen2024matrix}.

The results, presented in Figure~\ref{fig:mnist_shared}, demonstrate that our PSI-LinUCB algorithm achieves competitive performance with CBSCFD on MNIST while being significantly faster and outperforms all other baseline methods in both quality and computational efficiency. Results on CIFAR dataset are presented in Appendix, \Cref{app:experiments}, Figure~\ref{fig:cifar_shared}).
\begin{figure}[t!]
\centering
\includegraphics[width=\columnwidth]{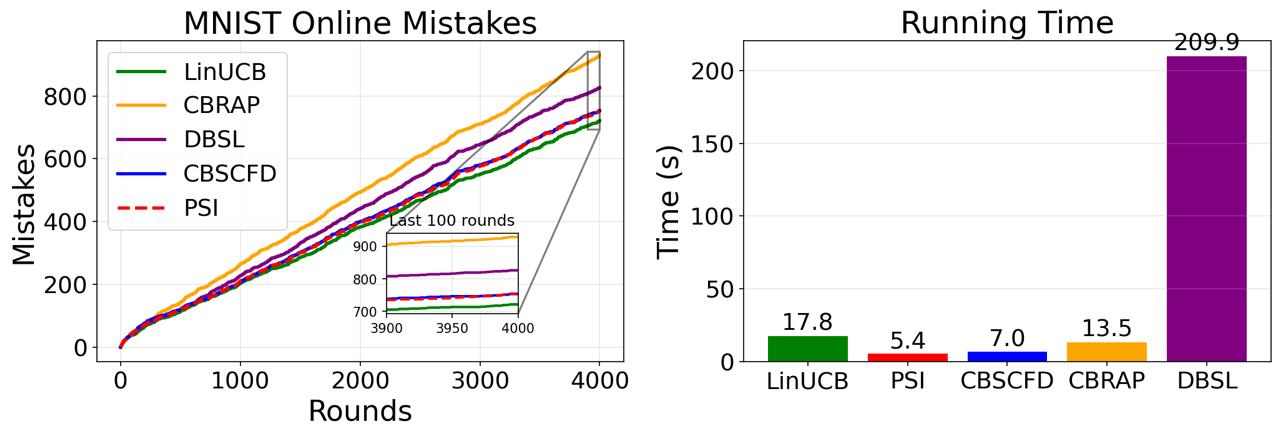}
\caption{Online classification results on the MNIST dataset.}
\label{fig:mnist_shared}
\end{figure}

\subsection{Approximation quality of the inverse matrix}\label{sec:approx_inverse}
Figure~\ref{fig:approx_inverse_mnist} shows the dependence of the relative approximation error of the inverse design matrix $A_{t}^{-1}$ on final iteration for different ranks (sketch sizes for CBSCFD). The approximation error of PSI-LinUCB and CBSCFD decreases as the rank increases. However, PSI-LinUCB consistently achieves lower error than CBSCFD for the same rank size.
The corresponding experiments on CIFAR (see Figure~\ref{fig:approx_inverse}) and real-world datasets (see Figures~\ref{fig:approx_inverse_instr}, \ref{fig:approx_inverse_mag}) are presented in Appendix, ~\Cref{app:experiments}.

 \begin{figure}[t]
    \centering
    \includegraphics[width=\columnwidth]{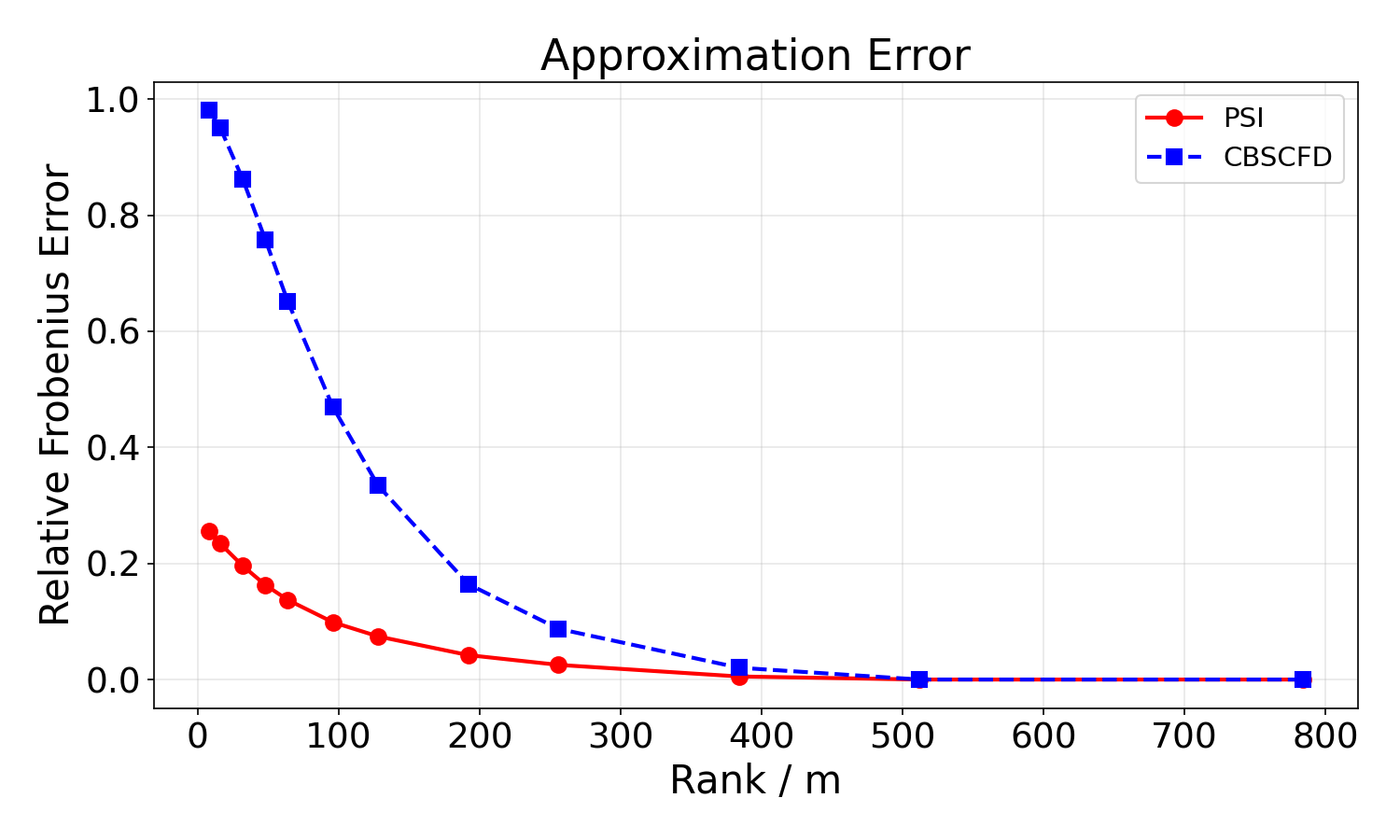}
    \caption{Approximation error of $A_{t}^{-1}$ on MNIST}
    \label{fig:approx_inverse_mnist}
\end{figure}

\subsection{Performance on real-world datasets}\label{sec:exp5}

This experiment compares PSI-LinUCB with CBSCFD, CBRAP and DBSLinUCB across multiple datasets. In the original methods, a shared design matrix $A_{t}$ is used for all arms; we extend this by training a separate matrix $A_{t,a}$ for each arm  $a$, as proposed in \cite{li2010contextual}. Additionally, CBSCFD and CBRAP support only rank-1 updates, while PSI-LinUCB generalizes to batch updates. The experiment follows the setup described in~\Cref{exp:dev}, for full details see~\Cref{appendix:tuning} in Appendix.

Table~\ref{tab:quality_metrics} presents quality metrics across three datasets (\textit{Amazon Magazine Subscriptions}, \textit{Amazon All Beauty}, \textit{Amazon Health}), and Table~\ref{tab:efficiency_metrics} reports training time, per-user prediction time and memory usage. CBRAP achieves quality comparable to PSI-LinUCB and CBSCFD but requires a larger sketch size $m$, resulting in higher memory and time costs. This explains inherently large training and evaluation time of CBRAP on the Beauty dataset. DBSLinUCB introduces additional tuning complexity, requiring joint search over the initial block size $l_0$ (which grows exponentially) and the error parameter $\varepsilon$, with interdependent effects. For example, DBSLinUCB fails on the Health dataset with all checked hyperparameters, see \Cref{tab:quality_metrics}. CBSCFD shows comparable quality to PSI-LinUCB, yet we point out its main drawback as \emph{instability of the quality with respect to sketch size $m$}.

\Cref{fig:cbscfd_ranks_amazon} and \Cref{fig:psi_ranks_amazon} show the relationship between the parameter $m$ and hit rate for CBSCFD and PSI-LinUCB, respectively. The results demonstrate significant variation in the optimal $m$ values across different real-world datasets, with unpredictable behavior: even small changes in $m$ lead to different results. At the same, for PSI-LinUCB the dependence of quality on approximation rank $r$ is monotone as the rank increases, which is desirable for tuning this (arguably, most important) hyperparameter in practice. We provide additional experiments on other datasets in the Appendix, \Cref{app:experiments} (see Figure~\ref{fig:cbscfd_ranks_other}).

\begin{table}[t!]
\centering
\small
\caption{Quality metrics on different datasets}
\label{tab:quality_metrics}
\setlength{\tabcolsep}{2pt}
\renewcommand{\arraystretch}{1.05}
\begin{tabular}{ll|ccc|c}
\hline
{Dataset} & {Algorithm} & {Hit@10} & {NDCG@10} & {MRR@10} & {Cov.} \\

\hline
\multirow{4}{*}{\rotatebox{90}{Magazine}} 
  & PSI-LinUCB & \textbf{0.531} & 0.430 & 0.398 & \textbf{0.038}\\
  & CBSCFD & 0.526 & \textbf{0.449} & \textbf{0.425} & 0.031 \\
  & CBRAP & 0.528 & 0.420 & 0.386 & 0.027 \\
  & DBSLinUCB & 0.523 & 0.438 & 0.411 & 0.028 \\
\hline

\multirow{4}{*}{\rotatebox{90}{Health}} 
  & PSI-LinUCB & \textbf{0.022} & \textbf{0.017} & \textbf{0.016} & \textbf{0.0016} \\
  & CBSCFD & 0.020 & 0.015 & 0.014 & 0.0012 \\
  & CBRAP & 0.020 & \textbf{0.017} & 0.016 & 0.0012 \\
  & DBSLinUCB & 0.001 & 0.001 & 0.001 & 0.0006 \\
\hline
\multirow{4}{*}{\rotatebox{90}{Beauty}} 
  & PSI-LinUCB & \textbf{0.013} & \textbf{0.011} & \textbf{0.010} & \textbf{0.0076} \\
  & CBSCFD & 0.012 & 0.009 & 0.009 & 0.0011 \\
  & CBRAP & 0.012 & \textbf{0.011} & \textbf{0.010} & 0.0011 \\
  & DBSLinUCB & \textbf{0.013} & \textbf{0.011}  & \textbf{0.010} & 0.0047 \\ 
\hline
\end{tabular}
\end{table}

\begin{table}[h]
\centering
\small
\caption{Computational efficiency on different datasets}
\label{tab:efficiency_metrics}
\setlength{\tabcolsep}{2pt}
\renewcommand{\arraystretch}{1.05}
\begin{tabular}{ll|cc|cc}
\hline
{Dataset} & {Algorithm} &{Train (s)} & {Eval (s)} & {Train (GB)} & {Eval (GB)} \\
\hline
\multirow{4}{*}{\rotatebox{90}{Magazine}} 
  & PSI-LinUCB & \textbf{10.3} & 0.145 & \textbf{1.4} & 1.6 \\
  & CBSCFD & 22.9 & 0.163 & 1.5 & 1.5 \\
  &CBRAP & 134.9 & 9.224 & 1.5 & 1.5 \\
  & DBSLinUCB & 24.6 & \textbf{0.102} & 1.5 & 1.8 \\
  \hline
\multirow{4}{*}{\rotatebox{90}{Health}} 
  & PSI-LinUCB & \textbf{25.9} & \textbf{1.24} & \textbf{6.9} & \textbf{8.8} \\
  & CBSCFD & 61.8 & 2.28 & 8.4 & 8.9 \\
  & CBRAP & 563.5 & 96.37 & 10.9 & 16.6 \\
  & DBSLinUCB & 518.8 & 1.5 & 7.3 & 9.1 \\
\hline
\multirow{4}{*}{\rotatebox{90}{Beauty}} 
  & PSI-LinUCB & \textbf{83.6} & \textbf{3.5} & 21.7 & 26.8 \\
  & CBSCFD & 106.5 & 6.01 & 22.4 & 27.2 \\
  & CBRAP & 1566.0 & 350.0 & 41.7 & 67.0 \\
  & DBSLinUCB & 290.8 & 3.6 & \textbf{19.1} & \textbf{24.8} \\
\hline
\hline
\end{tabular}
\end{table}

\subsection{Batch size and rank trade-off}\label{sec:exp6}

The batch size $B$ controls the frequency of PSI updates. Larger batches reduce training time but may slightly degrade quality. Figures~\ref{fig:health_batch} and~\ref{fig:beauty_batch} in Appendix show that training time decreases substantially with larger batches while hit rate remains stable. This enables flexible tuning: larger batches for latency-sensitive deployments, smaller batches for quality-critical scenarios.

\begin{figure}[t!]
  \centering
  \begin{subfigure}[b]{0.49\columnwidth}
    \includegraphics[width=\textwidth]{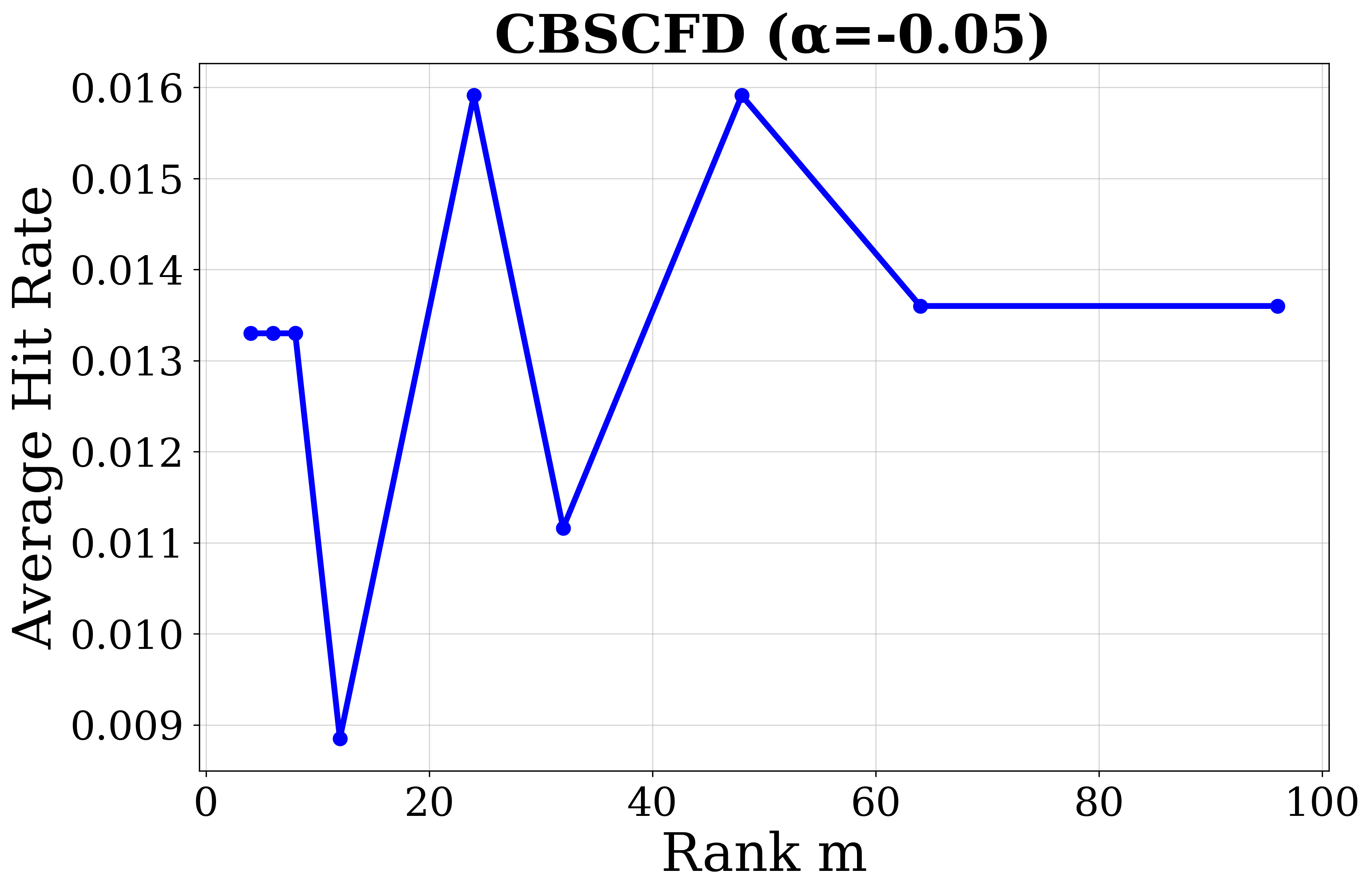}
    \caption{Amazon Health}
    \label{fig:cbscfd_health}
  \end{subfigure}
  \hfill
  \begin{subfigure}[b]{0.49\columnwidth}
    \includegraphics[width=\textwidth]{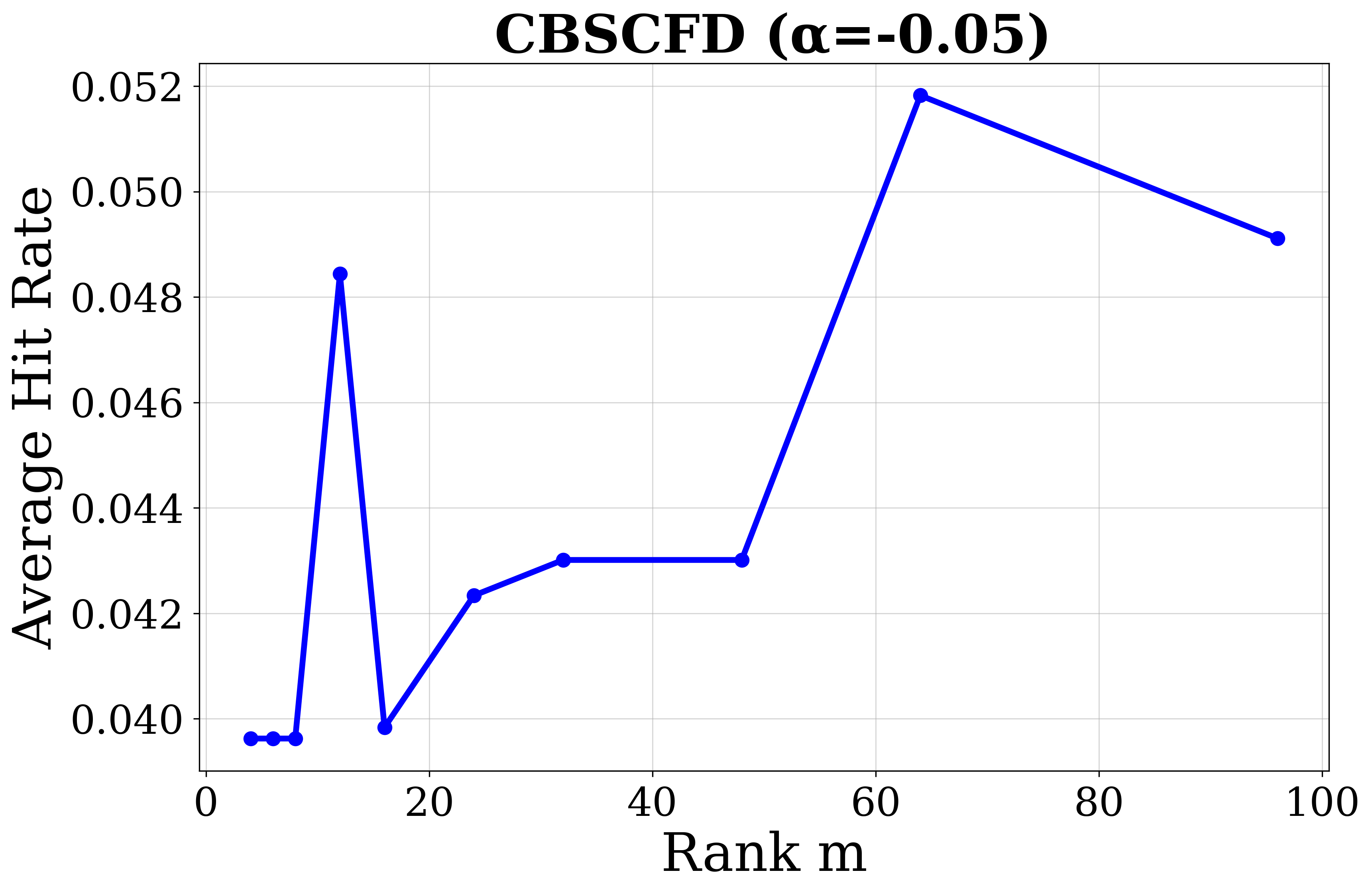}
    \caption{Amazon All Beauty}
    \label{fig:cbscfd_beauty}
  \end{subfigure}
  \caption{Average Hit Rate for different $m$ for CBSCFD.}
  \label{fig:cbscfd_ranks_amazon}
\end{figure}

\begin{figure}[t!]
  \centering
  \begin{subfigure}[b]{0.49\columnwidth}
    \includegraphics[width=\textwidth]{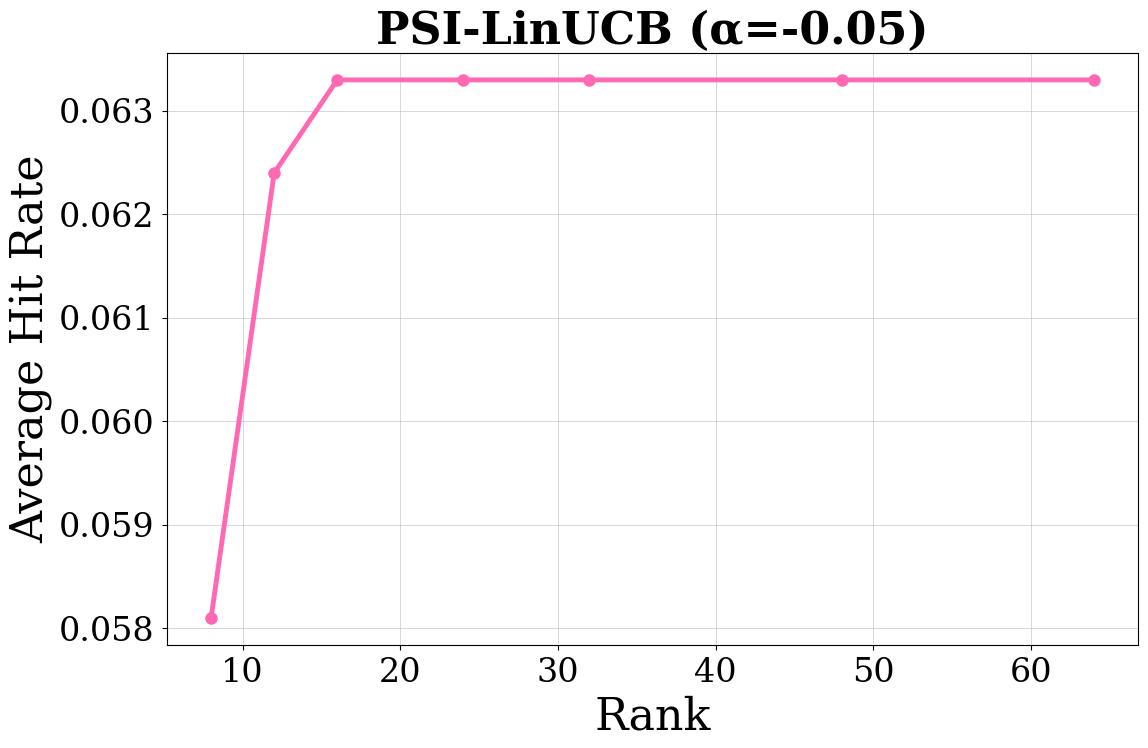}
    \caption{Amazon Health}
    \label{fig:psi_health}
  \end{subfigure}
  \hfill
  \begin{subfigure}[b]{0.49\columnwidth}
    \includegraphics[width=\textwidth]{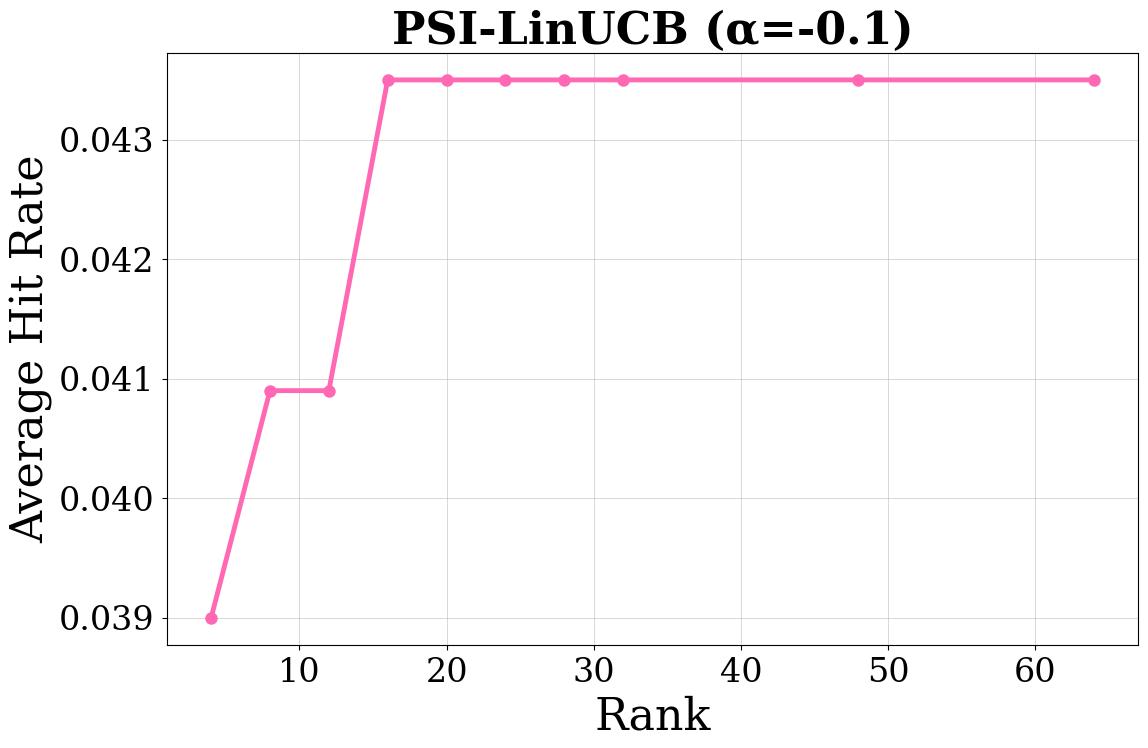}
    \caption{Amazon All Beauty}
    \label{fig:psi_beauty}
  \end{subfigure}
  \caption{Average Hit Rate for different $r$ for PSI-LinUCB.}
  \label{fig:psi_ranks_amazon}
\end{figure}

\section{Conclusion}
\label{sec:conclusion}
We presented \emph{PSI-LinUCB}, a scalable variant of LinUCB for large-scale contextual bandits. The method maintains a compact representation of the inverse regularized design matrix using a diagonal term and a low-rank correction, which allows efficient computation. Moreover, our method has an average complexity of updates of order $\mathcal{O}(dr)$ per candidate. A key direction for future work is to develop theoretical guarantees under direct approximations of the inverse regularized design matrix. To our knowledge, there are no regret guarantees for linear bandit algorithms in this setting. At the same time, as we show numerically, such algorithms might be preferable as compared to the classical sketching algorithms in particular applications.

\section*{Acknowledgements}
This research was supported in part through computational resources of HPC facilities at HSE University \cite{kostenetskiy2021hpc}.

\section*{Impact Statement}
This paper presents work whose goal is to advance the field of Machine Learning. There are many potential societal consequences of our work, none which we feel must be specifically highlighted here.

%%%%%%%%%%%%%%%%%%%%%%%%%%%%%%%%%%%%%%%%%%%%%%
%% Bibliography                             %%
%%%%%%%%%%%%%%%%%%%%%%%%%%%%%%%%%%%%%%%%%%%%%%
\newpage
\bibliography{refs-1}

%%%%%%%%%%%%%%%%%%%%%%%%%%%%%%%%%%%%%%%%%%%%%%
%% Appendixes                               %%
%%%%%%%%%%%%%%%%%%%%%%%%%%%%%%%%%%%%%%%%%%%%%%
\newpage
\begin{appendix}

\section{LinUCB with Batch Updates}
\label{sec:alg_linUCB_batch}
\begin{algorithm}[ht]
\caption{LinUCB}
\label{alg:LinUCB}
\begin{algorithmic}[1] 
\REQUIRE train\_data = $[batch_0, \dots, batch_{n-1}]$, regularization parameter $\lambda$
\STATE for each arm $a\in \mathcal{A}$ set $A_{0, a}= \lambda I$
    \FOR{t in \{0,\ldots,n-1\}}
    \STATE $X_{t,a} = []$, $R_{t,a} = []$
        \FORALL{$(u,a,r)$ in $batch_{t}$}
            \STATE $X_{t,a}.append\,[x_{t,a}]$  
            \STATE $R_{t,a}.append\,[r]$  
        \ENDFOR
        \FOR{each arm $a$ in $batch_t$}
            \STATE $A_{t+1,a} \leftarrow A_{t,a} + X_{t,a}^\top X_{t,a}$
            \STATE $b_{t+1,a} \leftarrow b_{t,a} + X_{t,a}^\top R_{t,a}$
            \STATE ${\theta}_{t+1,a} \leftarrow A_{t+1,a}^{-1} b_{t+1,a}$
        \ENDFOR
    \ENDFOR
\FOR{each arm $a \in \mathcal{A}$}
    \STATE Observe context $x_{t,a}$
    \STATE Compute $\text{UCB}_{t,a} =  {\theta}_{t,a}^\top {x}_{t,a} + \alpha \sqrt{{x}_{t,a}^\top {A}_{t,a}^{-1} {x}_{t,a}}$
\ENDFOR
\STATE \textbf{return} $\arg\max_{a \in \mathcal{A}} \text{UCB}_{t,a}$
\end{algorithmic}
\end{algorithm}

\section{Related Work and Baseline Selection}
\label{sec:related-work}
\subsection{Scalable variants of LinUCB}
Although LinUCB \cite{li2010contextual} is a widely used algorithm in recommendation systems, the time and memory requirements increases with the dimension of the context $d$ and the number of items, since it is necessary to store and invert the matrix $A_a\in \mathbb{R}^{d\times d}$ for each item $a$.
To address scalability constraints, several extensions and variants of the LinUCB algorithm have been proposed in the
literature.

A widely used approach to accelerate LinUCB is to apply rank-1 updates of the ridge-regularized design matrix $A_a$ via the Sherman–Morrison identity~\cite{Angioli2025Efficient, Ciucanu2022Implementing, Wang2022Dynamic, Yan2025CoCoB, Ozbay2024Comparative, zenati2022efficient}.
These updates reduce the per-round update complexity from $\mathcal{O}(d^3)$ to $\mathcal{O}(d^2)$ while maintaining exact parameter estimates, but still require storing a full $d \times d$ matrix for each arm.

Another line of research aims to improve the efficiency of LinUCB through matrix sketching techniques. Early work \cite{kuzborskij2019efficient} demonstrated that LinUCB can be efficiently implemented using the Frequent Directions (FD) sketching method, reducing the per-round update time from $\mathcal{O}(d^2)$ to $\mathcal{O}(md)$, where $m$ is the sketch size and $d$ is the feature dimension. However, applying FD to contextual bandits has certain drawbacks. In particular, the FD sketching method violate the positive definite monotonicity design matrices $A_{t,a}$, which may affect stability and theoretical guarantees. To address this limitation, a Spectral Compensation Frequent Directions (SCFD) method and its adaptation (CBSCFD) for high-dimensional contextual bandit were proposed \cite{chen2020efficient}, which preserves positive definiteness while maintaining the same $\mathcal{O}(md)$ computational and memory complexity.
More recently, adaptive sketching techniques such as Dyadic Block Sketching \cite{wen2024matrix} have been introduced to dynamically adjust the sketch size, ensuring a per-round update complexity of $\mathcal{O}(d l)$, where $l$ is the current sketch size. 
This adaptive strategy prevents excessive spectral loss and avoids linear regret when the spectrum of the design matrix decays slowly.  

Another perspective on improving the efficiency of linear contextual bandits is through dimensionality reduction via random projection. For example, the Contextual Bandits via Random Projection (CBRAP) algorithm \cite{yu2017cbrap}, address the challenges of high-dimensional contexts by mapping the original $d$-dimensional features to a lower $m$-dimensional subspace. 
This reduces the per-round update complexity from $\mathcal{O}(d^2)$ to $\mathcal{O}(md + m^3)$. 

One more approach to tackle the limitations of the LinUCB algorithm is to approximate each design matrix $A_a$ along its diagonal as proposed in Diag-LinUCB \cite{yi2023online}. This allows scalable online updates running as $\mathcal{O}(d)$ per round in terms of both time and memory. As the authors demonstrate, in the specific contexts $x_{ua}$ such updates are sufficient to prevent the loss of model quality during training. 

\subsection{Baselines used in our experiments}
To evaluate PSI-LinUCB, we select baselines that cover both exact implementations and implementations using covariance matrix approximation, which reduce either the effective rank of the matrix or the feature dimension.

As exact implementations, we report results for Batch LinUCB \citep{li2010contextual}, which matches batched training protocol provided in \Cref{alg:LinUCB} and LinUCB Classic with Sherman--Morrison rank-one inverse updates \citep{Angioli2025Efficient}. These baselines provide a useful quality reference and illustrate the memory and computational cost of exact updates in high dimensions.

To represent design-matrix compression via sketching, we use CBSCFD \citep{chen2020efficient}, which combines FD with an additional correction and  maintains a low-rank sketch of the design matrix. CBSCFD is widely used as a robust sketching baseline for high-dimensional linear contextual bandits. We additionally include Dyadic Block Sketching (DBSLinUCB) \citep{wen2024matrix} as a recent adaptive sketching approach that dynamically adjusts sketch size. 

To cover feature-space compression, we include CBRAP \citep{yu2017cbrap}, which applies random projections before performing LinUCB-style updates. This baseline is conceptually different from sketching the design matrix and provides a typical alternative for dimensionality reduction.

Some methods are closely related but are not included as primary baselines. We do not include FD \citep{kuzborskij2019efficient} in the main comparison because CBSCFD \citep{chen2020efficient}   combines FD with an additional correction and provides a stronger and more robust representative within the same sketching family. We also omit Diag-LinUCB \citep{yi2023online} because it suggests to learn context representations (using non-linear transformations) followed by diagonal approximation and thus the comparison with original LinUCB is uninformative.

\section{Experiments}
\label{app:experiments}
\subsection{Hyperparameter Selection}
\label{appendix:tuning}

\textit{Context dimensionality $d$}

The optimal values of $r'$ were selected using the scree plot, which displays singular values in descending order. The resulting context sizes are:
\begin{itemize}
   \item {MovieLens 1M}: $r' = 13$, $d = 169$;
  \item {Magazine Subscriptions}: $r' = 41$, $d = 1681$;
  \item {Health \& Personal Care}: $r' = 37$, $d = 1369$;
  \item {All Beauty}: $r' = 51$, $d = 2601$.
\end{itemize}

\textit{Regularization $\alpha$}

We used grid search to find the optimal exploration parameter $\alpha$. First, a broad grid $\alpha \in \{-10,-8,\dots,10\}$ was evaluated, followed by a refined search over $\alpha \in \{-1, -0.5, -0.3, -0.2, -0.1, -0.05, -0.03, -0.01, 0.1\}$. The optimal $\alpha$ for LinUCB coincided with that for PSI-LinUCB.

\textit{Rank and sketch size}

PSI-LinUCB rank, CBSCFD/CBRAP parameter $m$, DBSLinUCB block size $l_0$ (which is not fixed but grows
exponentially during learning) and the error parameter $\varepsilon $, was explored over the grid $\{b \cdot 2^i \mid i \in \{1,\dots,8\}, b \in \{2,3\}\}$.

\textit{Feature extraction}

Given an interaction matrix $A \in \mathbb{R}^{m \times n}$, we apply truncated SVD: $A \approx U \Sigma V^{\top}$. Item features are obtained from $V^\top \in \mathbb{R}^{n \times r}$, and user features from $A \cdot V^{\top} \in \mathbb{R}^{m \times r}$. Features are extracted from warm-up data, cold users lacking valid context representations are excluded from evaluation. Note that cold-start users is not inherent to bandits, as soon as there are available contexts, the algorithm can be run, the restriction arises from our feature construction SVD-based pipeline and does not alter the fundamental bandit framework.
Then we construct a context ${x}_{t, a} = \text{vec}(x_u {x}_a^{\top}) \in \mathbb{R}^{d_u \times d_a}$, where $x_u$ and $x_a$ are user and item feature vectors. In our experiments, we set $d_u = d_a = r'$, yielding $d = (r')^2$.

% The LinUCB algorithm is implemented with a batch processing approach for scalable recommendation systems.  The main goal of the developed algorithm is to ensure scalability with respect to high-dimensional contextual features.
% Features are extracted using SVD decomposition of user-item interaction matrix, the rank $r'$ we choose via spectral analysis from the data available during the warm-up data. 
% We construct a joint feature vector ${x}_{ua} = \text{vec}(x_u {x}_a^{\top}) \in \mathbb{R}^{d_u \times d_a}$, where $x_u$ and $x_a$ are user and item feature vectors. In our experiments, we set $d_u = d_a = r'$, yielding $d = (r')^2$. 
% Cold users lacking valid context representations are excluded from evaluation. Note that cold-start  users is not inherent to bandits but arises solely from our feature construction pipeline and does not alter the fundamental bandit framework.

\textit {Data processing}

Training time measures initial warm-up training; evaluation time is the average prediction duration per user, averaged across test months. The interaction dataset is grouped by arms, with each batch corresponding to a fixed number of interactions for a single arm. This scheme is applied uniformly across all algorithms for fair comparison. The batch-to-rank ratio (see~Remark $\ref{rem}$) in PSI-LinUCB corresponds to the compression frequency factor of 2 in CBSCFD; however, batch size does not affect final model state or quality for CBSCFD and CBRAP, as these methods perform incremental per-sample updates.

\begin{figure}[t!]
  \centering
  \begin{subfigure}[b]{0.49\columnwidth}
    \includegraphics[width=\textwidth]{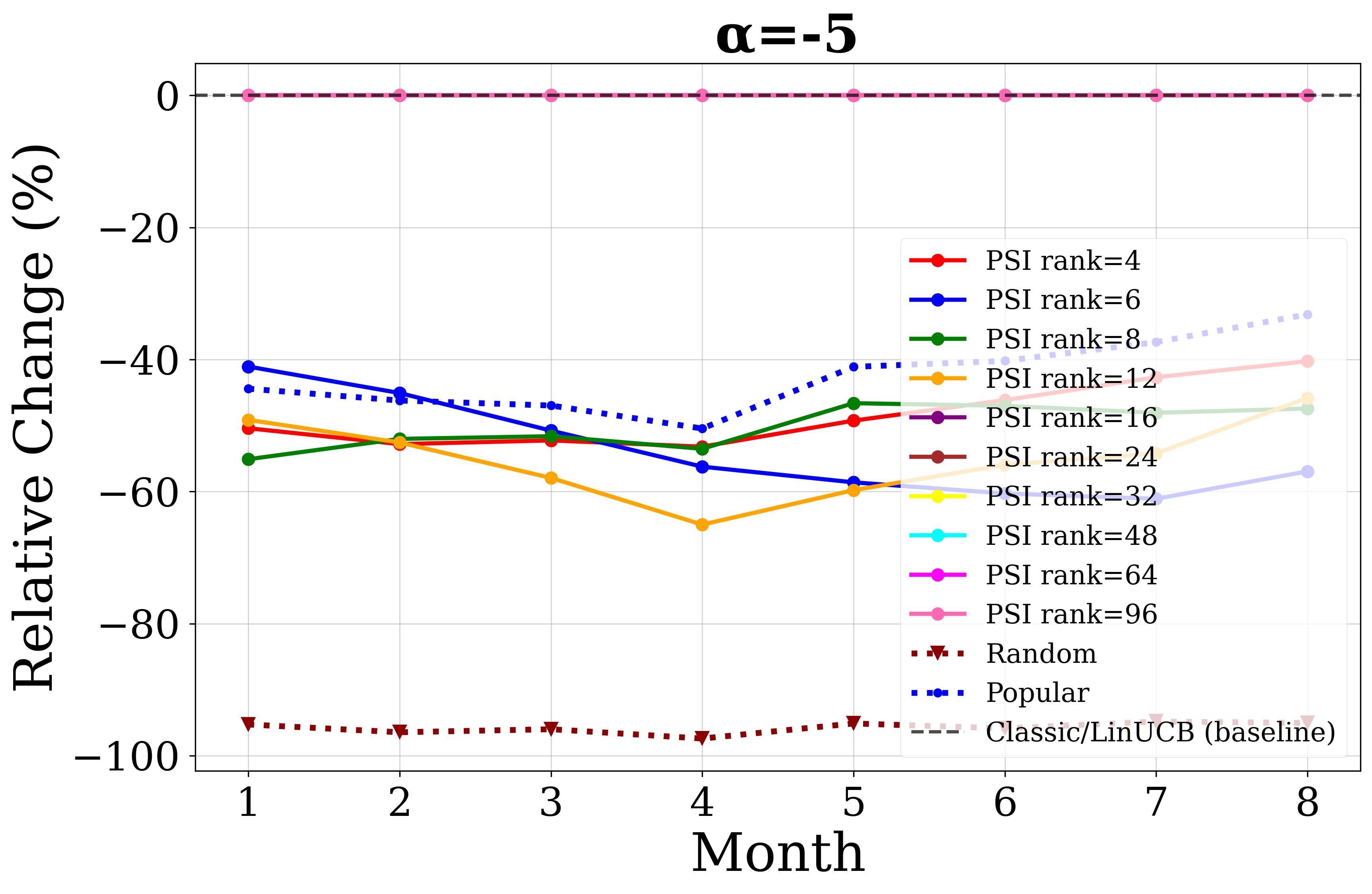}
    \caption{MovieLens 1M}
    \label{fig:quality_movielens}
  \end{subfigure}
  \hfill
  \begin{subfigure}[b]{0.49\columnwidth}
    \includegraphics[width=\textwidth]{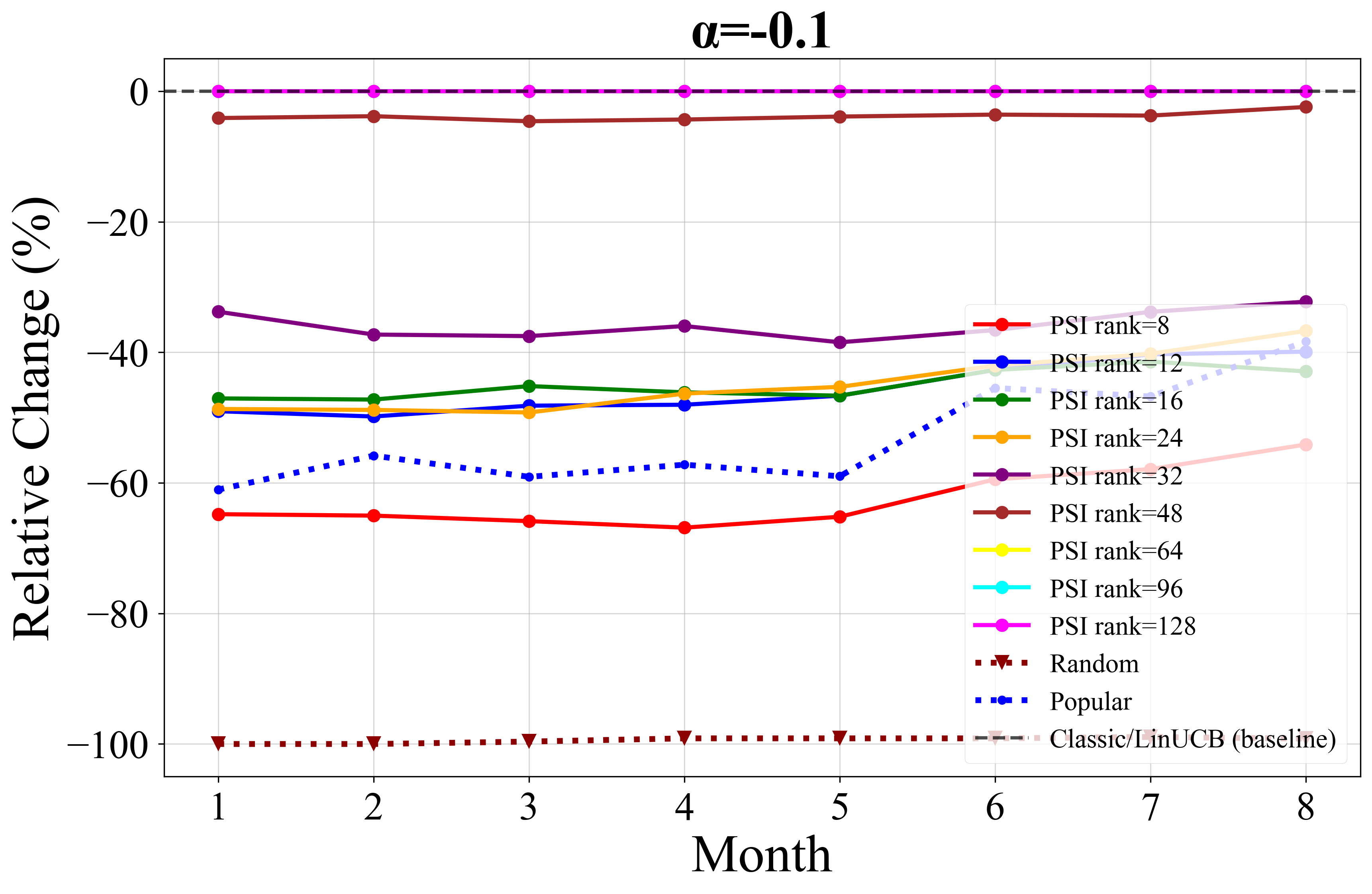}
    \caption{Magazine Subscriptions}
    \label{fig:quality_music}
  \end{subfigure}
  \caption{Average Hit Rate for different $r$ for PSI-LinUCB.}
  \label{fig:quality_comparison_other}
\end{figure}

\begin{figure}[t!]
    \centering
    \includegraphics[width=1\columnwidth]{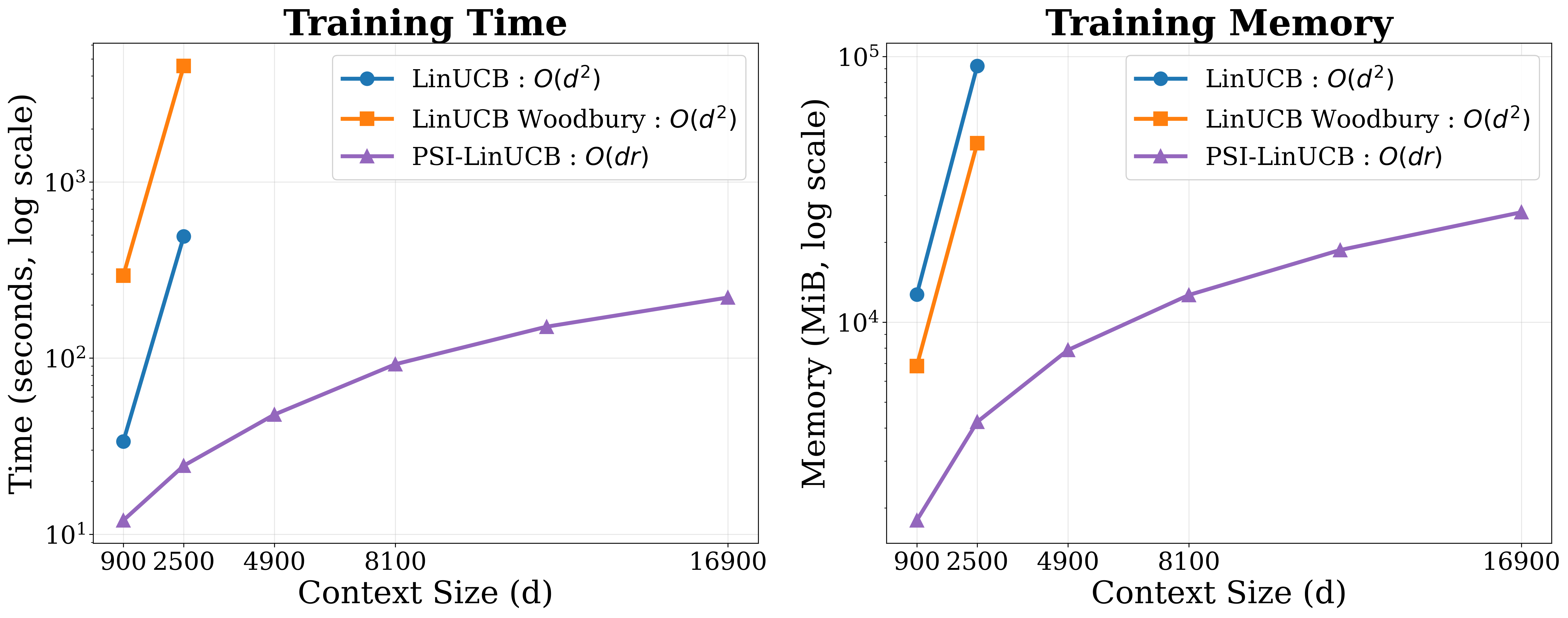}
    \caption{Performance comparison across different context sizes on Amazon Health dataset.}
    \label{fig:context_mag_train}
\end{figure}

\begin{figure}[t!]
  \centering
  \begin{subfigure}[b]{0.49\textwidth}
    \includegraphics[width=\textwidth]{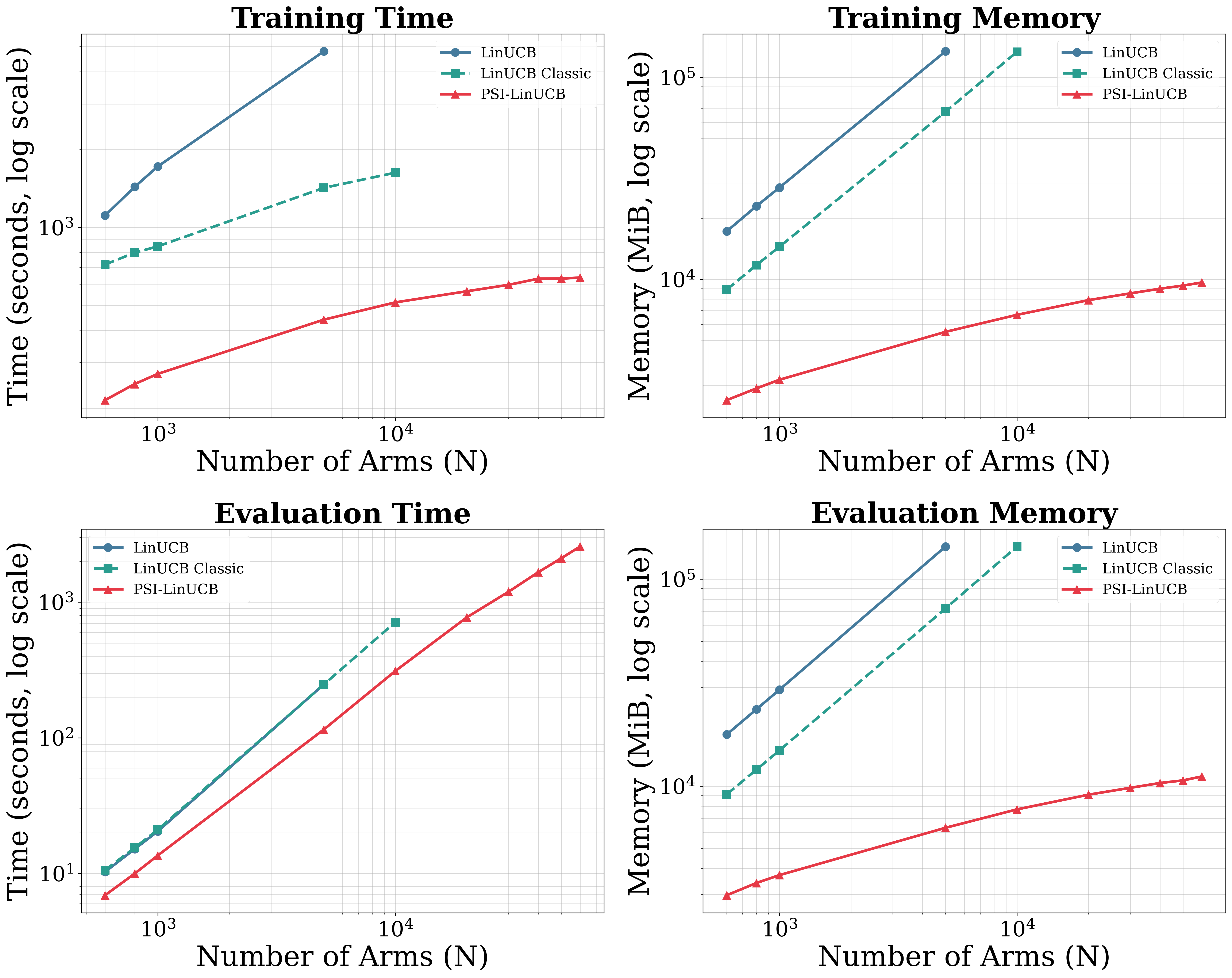}
    \caption{Health \& Personal Care}
    \label{fig:arms_health}
  \end{subfigure}%
  \hfill
   \begin{subfigure}[b]{0.49\textwidth}
    \includegraphics[width=\textwidth]{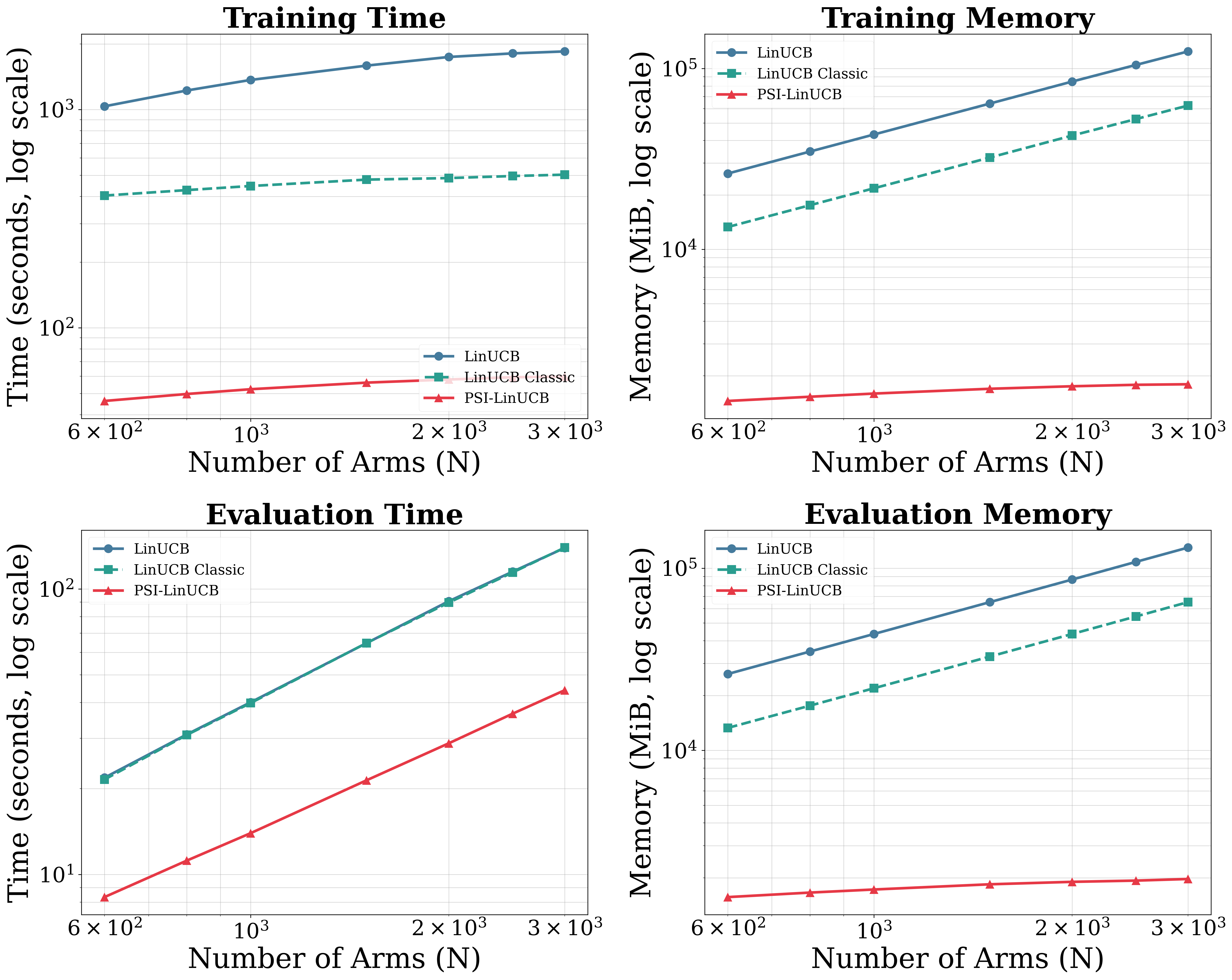}
    \caption{Magazine Subscriptions}
    \label{fig:arms_mag}
  \end{subfigure}
  
  \vspace{2mm}

  \begin{subfigure}[b]{0.48\textwidth}
    \includegraphics[width=\textwidth]{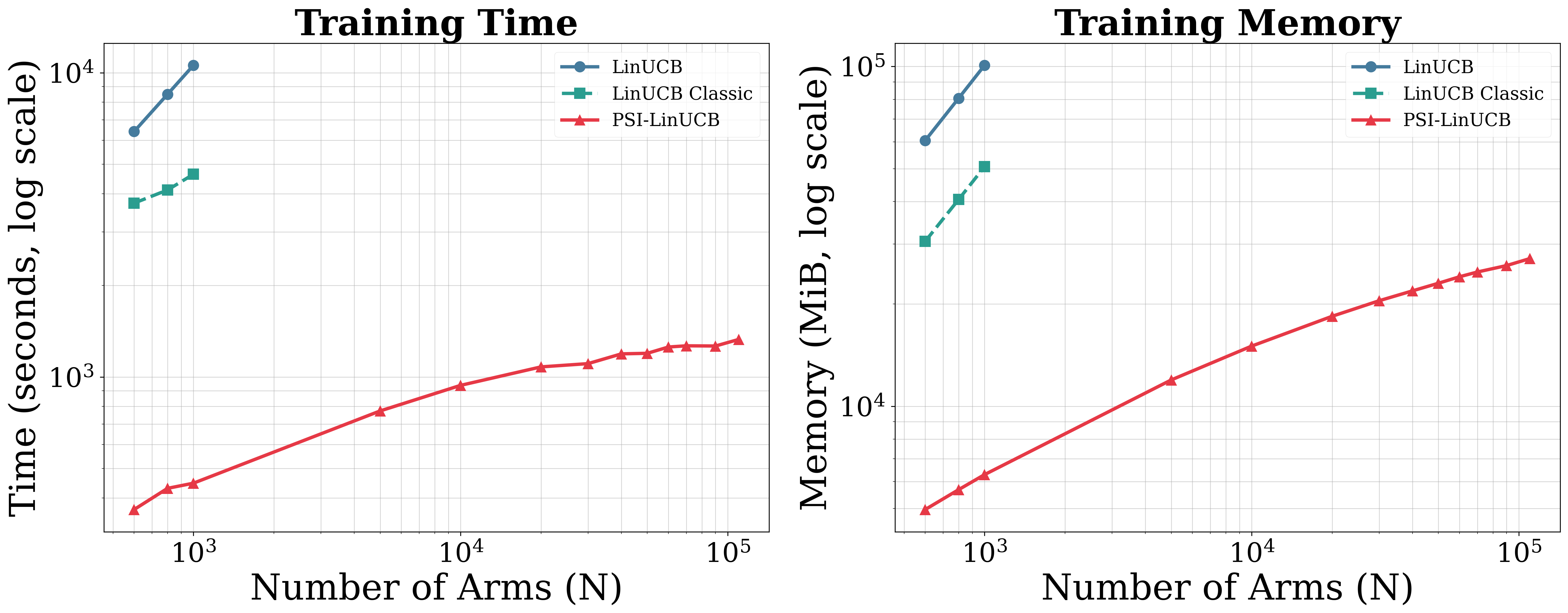}
    \caption{Beauty}
    \label{fig:arms_beauty_train}
  \end{subfigure}
  \caption{Algorithm scaling with number of arms on different datasets.}
  \label{fig:arms_scaling}
\end{figure}

\begin{figure}[t]
  \centering
  \includegraphics[width=\columnwidth]{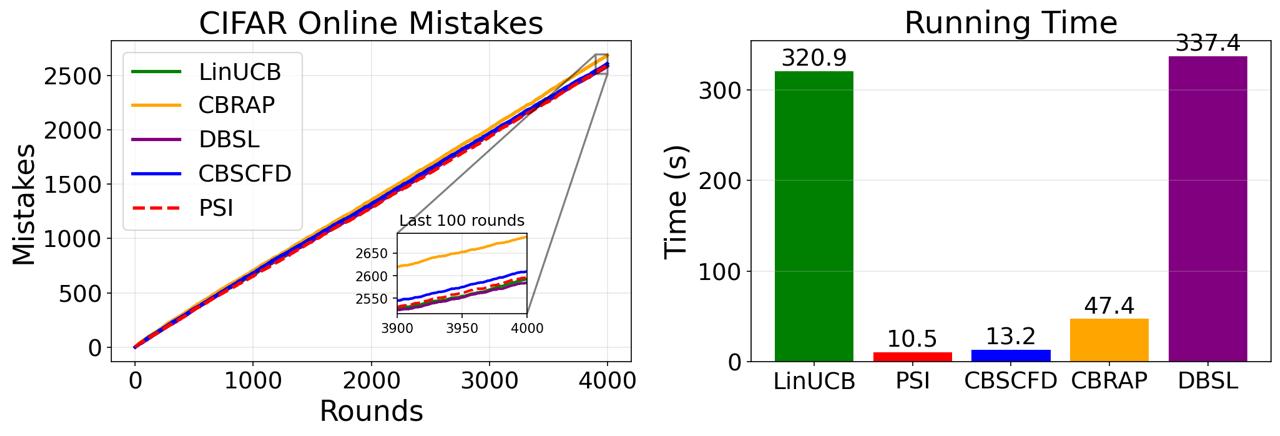}
  \caption{Online classification results on CIFAR-10.}
  \label{fig:cifar_shared}
\end{figure}

\begin{figure}[t]
  \centering
  \includegraphics[width=\columnwidth]{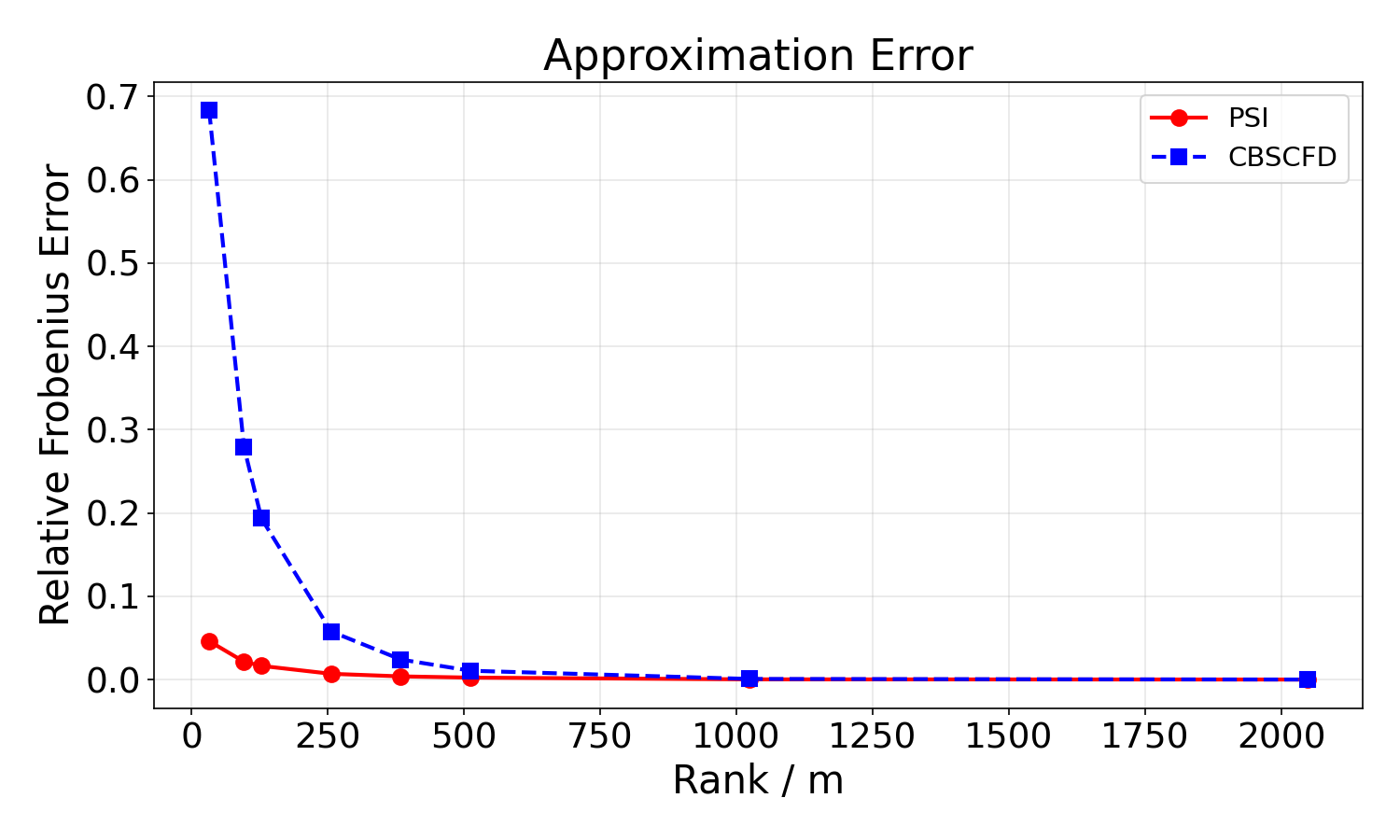}
  \caption{Approximation error of $A^{-1}$ on CIFAR-10.}
  \label{fig:approx_inverse_cifar}
\end{figure}

\begin{table}[t!]
\centering
\small
\caption{Summary of datasets for online classification}
\label{tab:cifar_mnist}
\begin{tabular}{lcccc}
\hline
\textbf{Dataset} & \textbf{\#Samples} & \textbf{\#Features} & \textbf{\#Classes} &  \\
\hline
MNIST & 60000 & 784 & 10 \\
CIFAR-10 & 50000 & 3072 & 10 \\
\hline
\end{tabular}
\end{table}

\begin{figure}[t!]
  \centering
  \begin{subfigure}[b]{0.49\textwidth}
    \includegraphics[width=\textwidth]{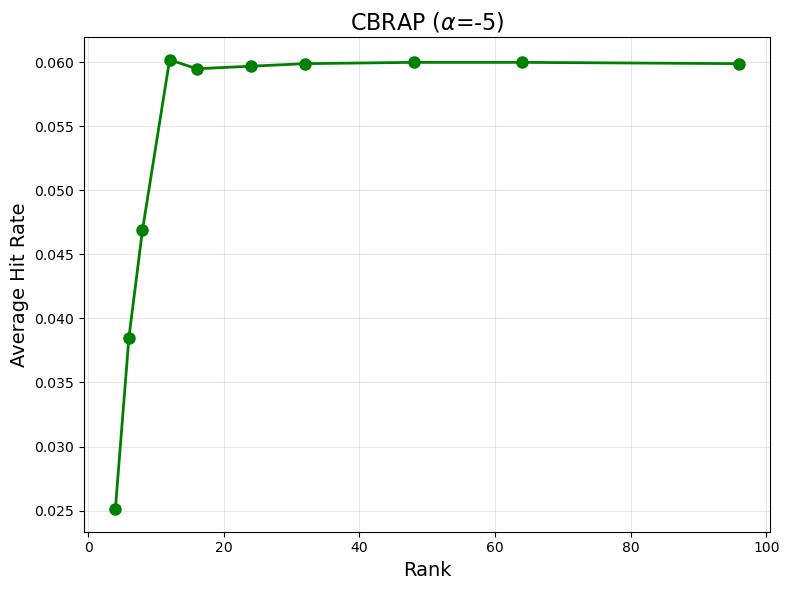}
    \caption{MovieLens 1M}
    \label{fig:cbrap_movielens}
  \end{subfigure}
  \hfill
  \begin{subfigure}[b]{0.49\textwidth}
    \includegraphics[width=\textwidth]{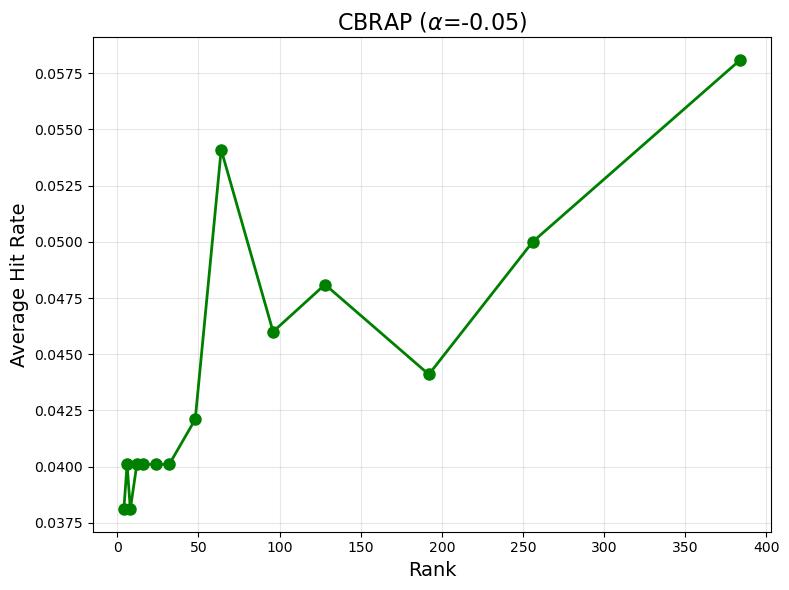}
    \caption{Amazon All Beauty}
    \label{fig:cbrap_beauty}
  \end{subfigure}
  
  \vspace{2mm}
  
  \begin{subfigure}[b]{0.49\textwidth}
    \includegraphics[width=\textwidth]{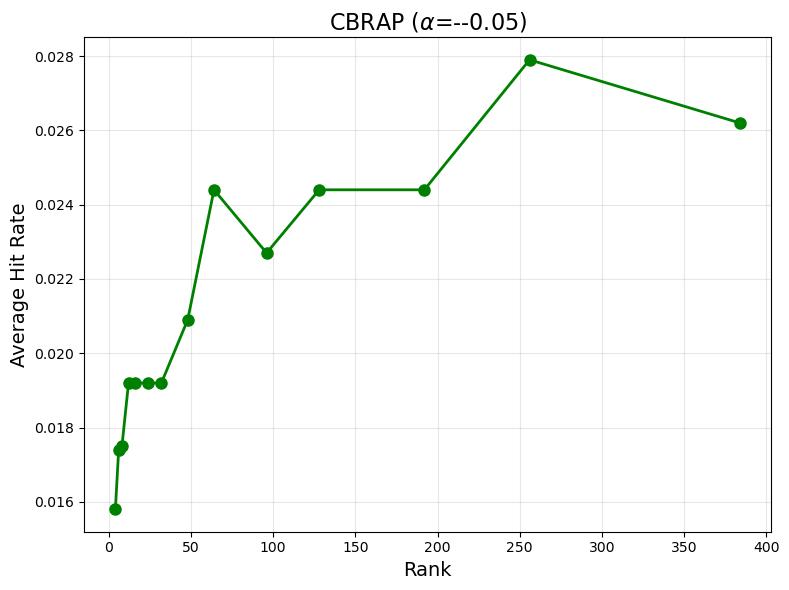}
    \caption{Amazon Health}
    \label{fig:cbrap_health}
  \end{subfigure}
  \hfill
  \begin{subfigure}[b]{0.49\textwidth}
    \includegraphics[width=\textwidth]{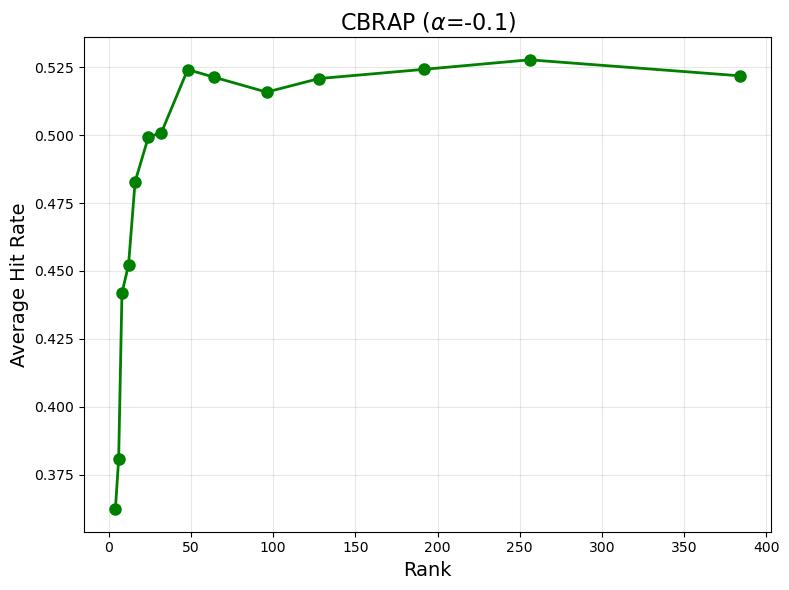}
    \caption{Magazine Subscriptions}
    \label{fig:cbrap_music}
  \end{subfigure}
  \caption{Average Hit Rate for different $m$ values for CBRAP.}
  \label{fig:cbrap_ranks}
\end{figure}

\begin{figure}[t!]
  \centering
  \begin{subfigure}[b]{0.49\textwidth}
    \includegraphics[width=\textwidth]{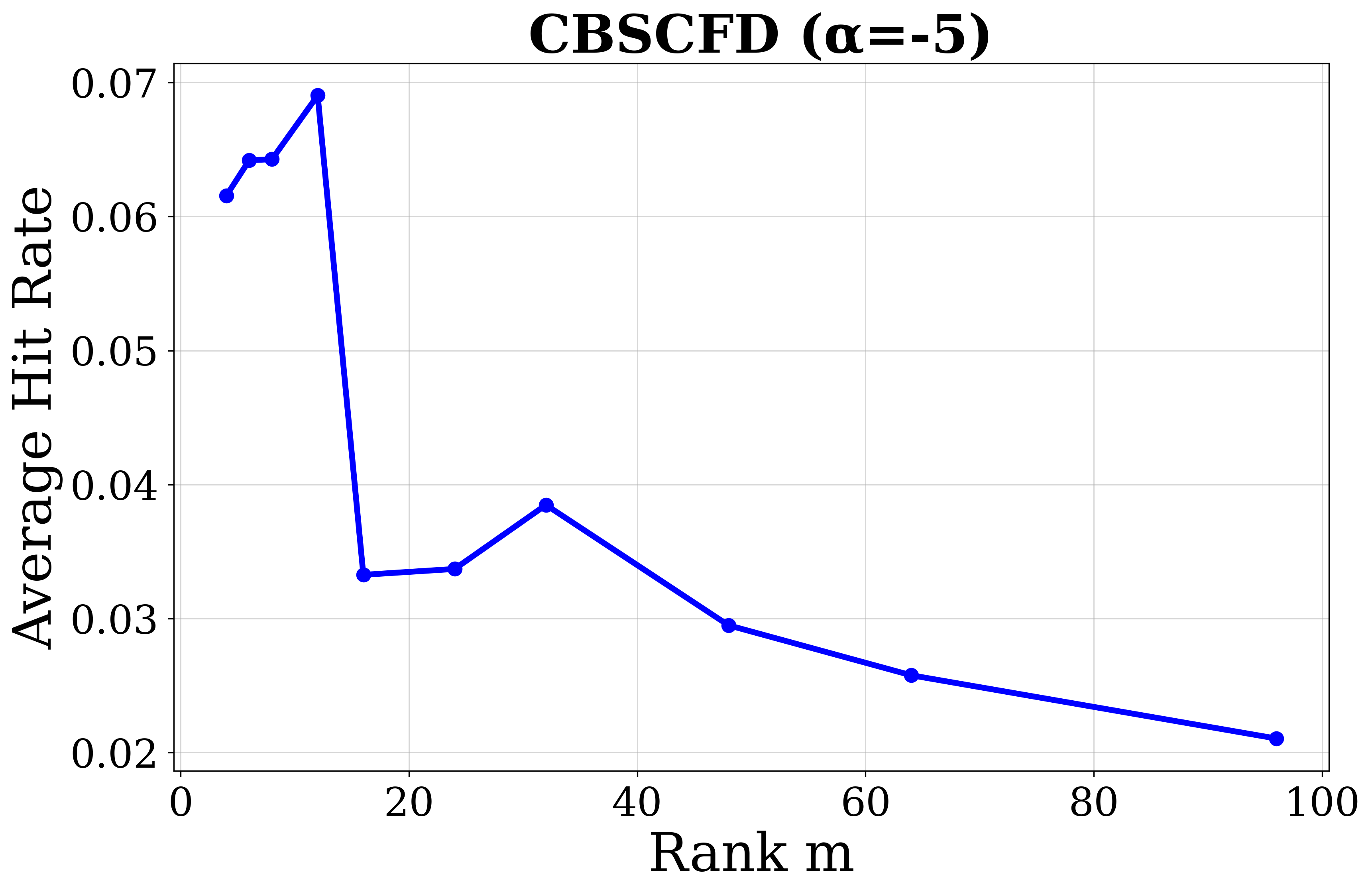}
    \caption{MovieLens 1M}
    \label{fig:cbscfd_movielens}
  \end{subfigure}
  \hfill
  \begin{subfigure}[b]{0.49\textwidth}
    \includegraphics[width=\textwidth]{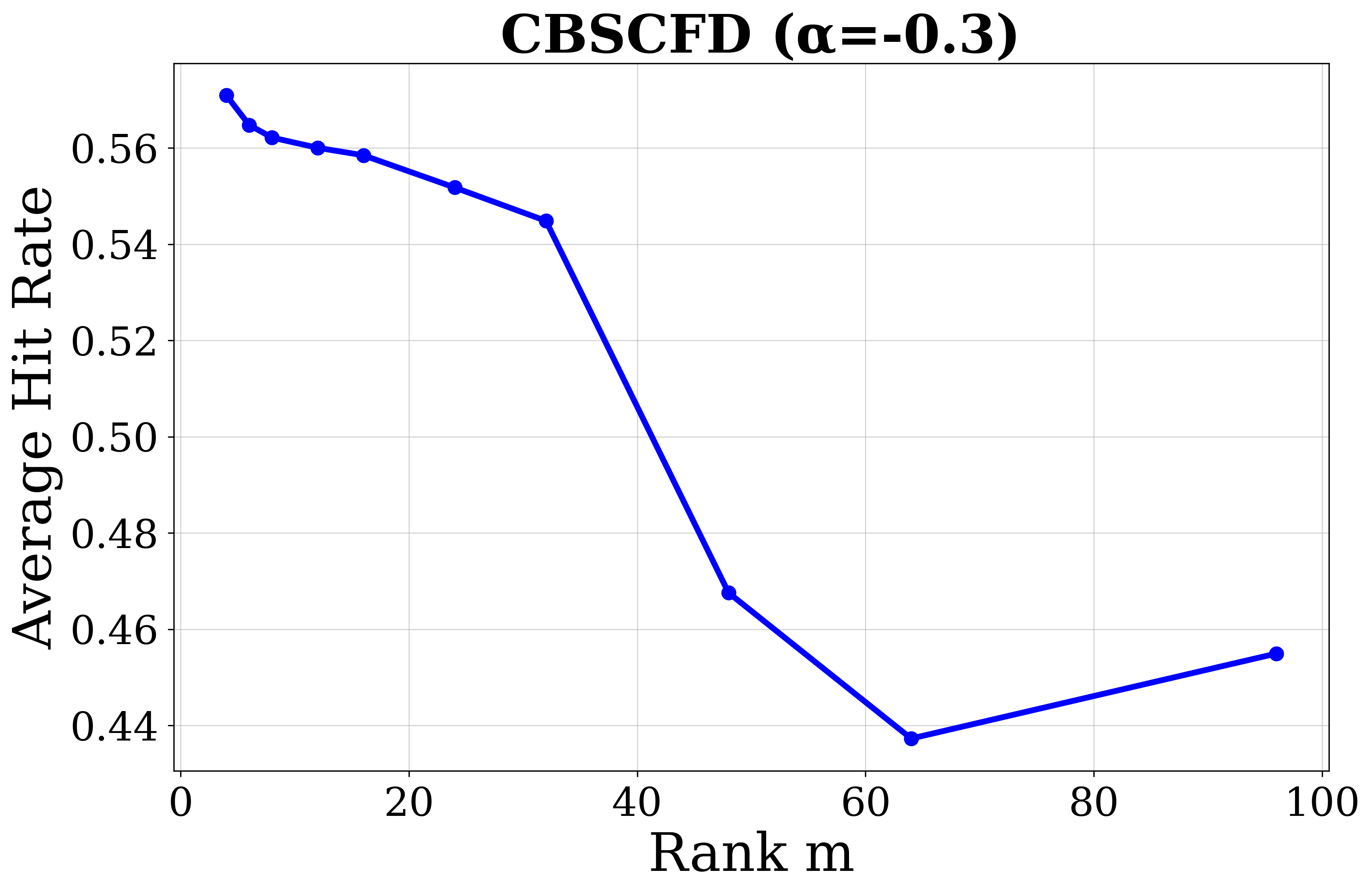}
    \caption{Magazine Subscriptions}
    \label{fig:cbscfd_music}
  \end{subfigure}
  \caption{Average Hit Rate for different $m$ for CBSCFD.}
  \label{fig:cbscfd_ranks_other}
\end{figure}

\begin{figure}[t!]
  \centering
  \begin{subfigure}[b]{0.49\textwidth}
    \includegraphics[width=\textwidth]{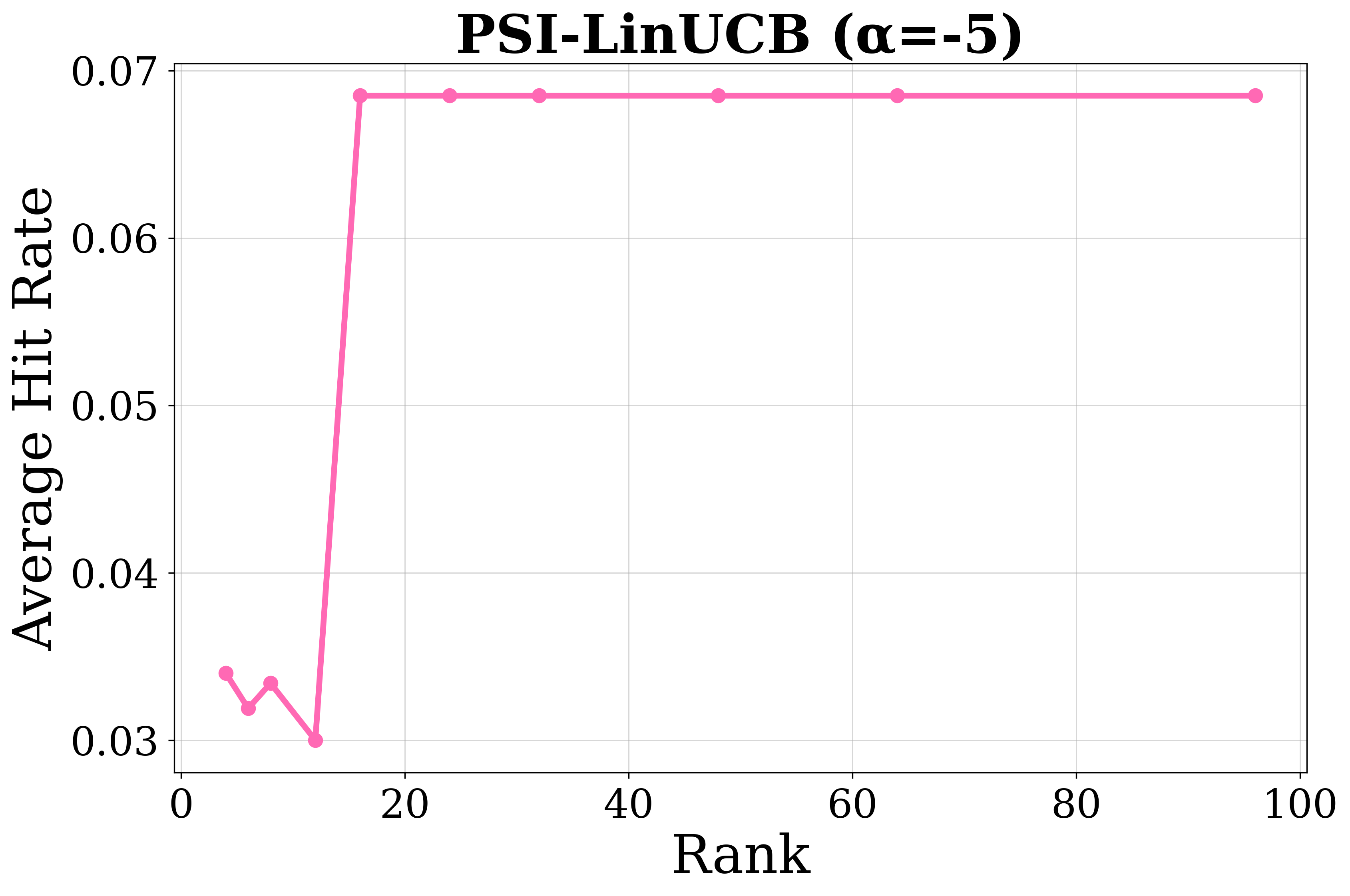}
    \caption{MovieLens 1M}
    \label{fig:psi_movielens}
  \end{subfigure}
  \hfill
  \begin{subfigure}[b]{0.49\textwidth}
    \includegraphics[width=\textwidth]{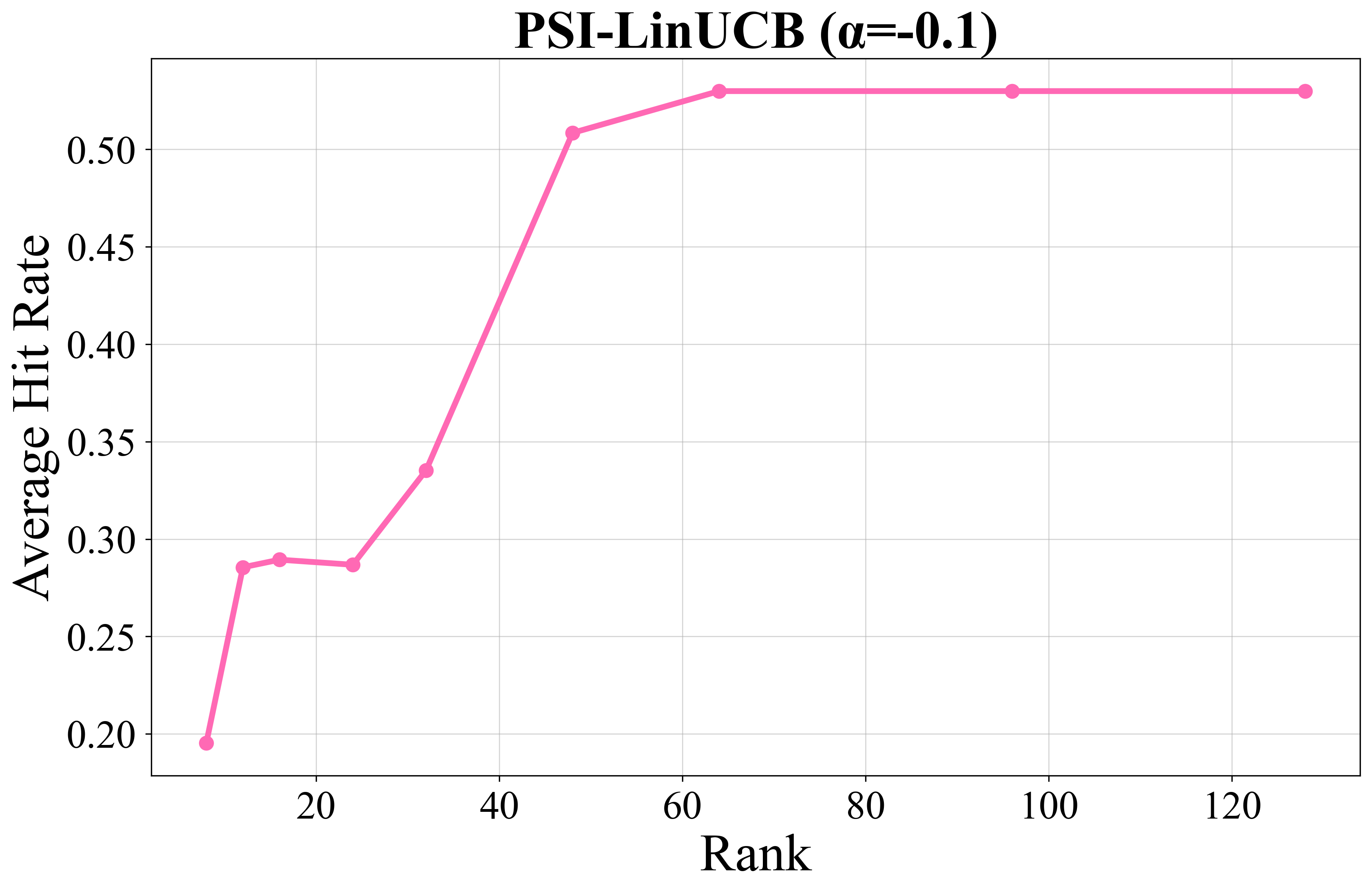}
    \caption{Magazine Subscriptions}
    \label{fig:psi_music}
  \end{subfigure}
  \caption{Average Hit Rate for different rank values for PSI-LinUCB.}
  \label{fig:psi_ranks_other}
\end{figure}

\begin{figure}[t]
    \centering
    \begin{subfigure}[b]{0.48\textwidth}
        \centering
        \includegraphics[width=\textwidth]{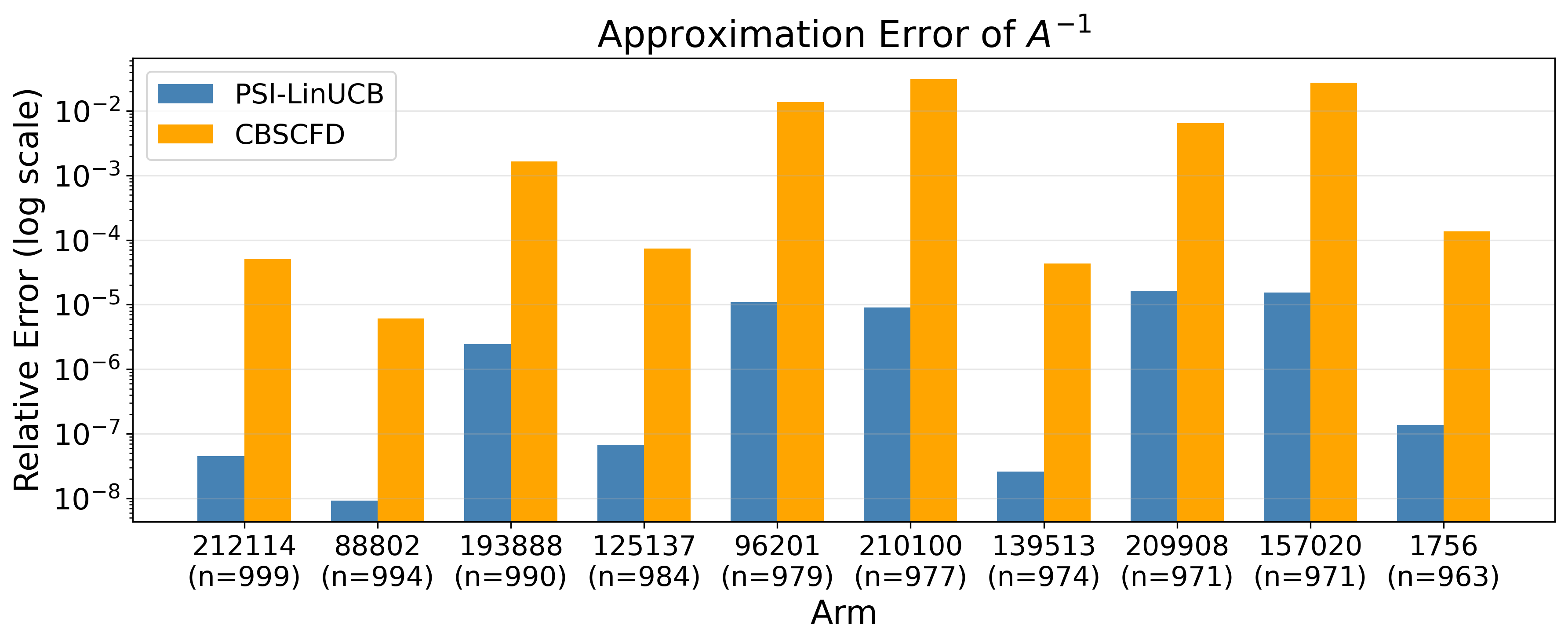}
        \caption{Amazon Health}
        \label{fig:approx_inverse_instr}
    \end{subfigure}
    \hfill
    \begin{subfigure}[b]{0.48\textwidth}
        \centering
        \includegraphics[width=\textwidth]{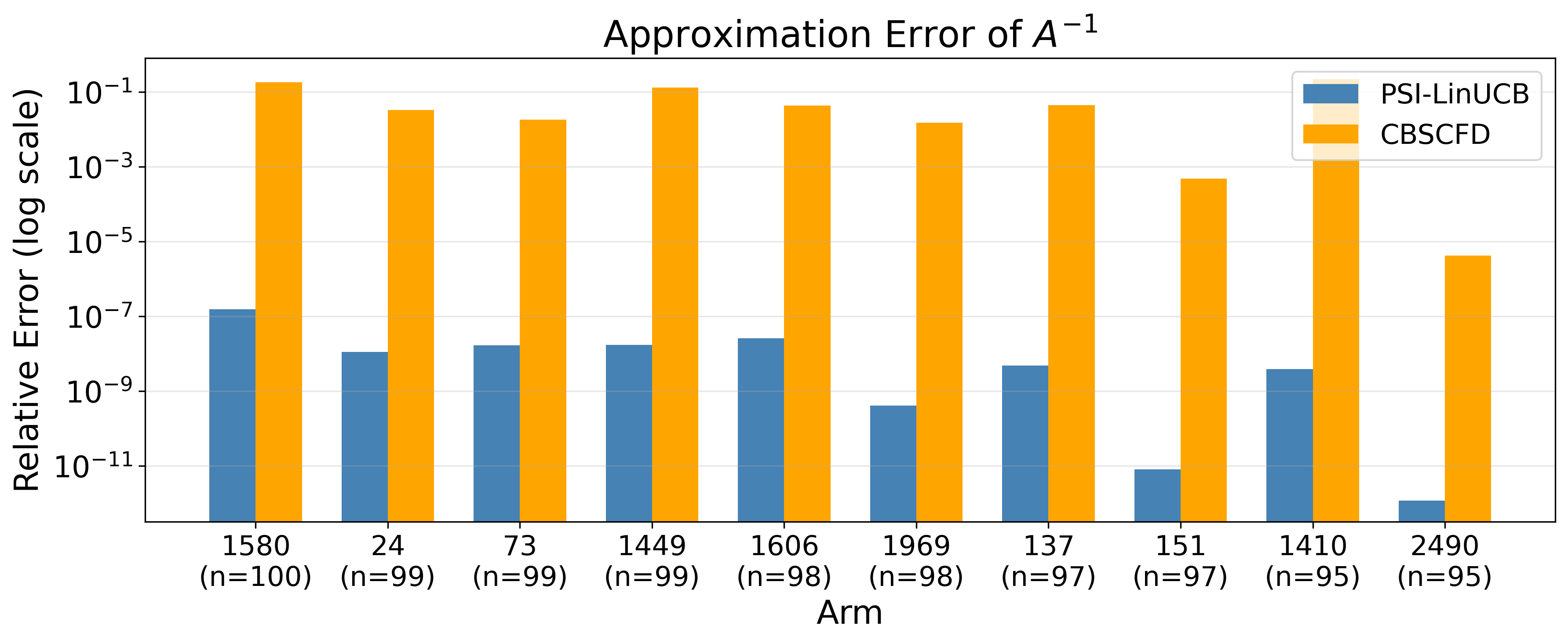}
        \caption{Magazine Subscriptions}
        \label{fig:approx_inverse_mag}
    \end{subfigure}
    \caption{Relative approximation error of $A^{-1}$ on real-world datasets.}
    \label{fig:approx_inverse}
\end{figure}

\subsection{Setting of online classification}\label{appendix:exp5}

The hyperparameters $\beta$ and $\lambda$ are selected via grid search over $\{10^{-4}, 10^{-3}, \ldots, 1\}$ and $\{2 \times 10^{-4}, 2 \times 10^{-3}, \ldots, 2 \times 10^{4}\}$, respectively, following the experimental protocol of~\cite{chen2020efficient}. 
For our PSI-LinUCB algorithm, the rank parameter $r$ and for sketching-based methods (CBSCFD, CBRAP) the optimal sketch size $m$ are selected via cross-validation over a grid ranging from $10$ to $100$ with step $10$ and $\{200, 300, 400, 500\}$ by minimizing the number of the average online mistakes on a validation run.

All experiments are repeated $20$ times and we report the average cumulative mistakes.

This dual applicability represents a practical advantage of our approach: the same algorithmic framework can be deployed both in the per-arm setting commonly used in real-world recommendation systems and in the shared model setting typical of experimental benchmarks. This allows for detailed theoretical analysis in controlled settings while maintaining direct applicability to production environments.

In addition to the baselines from~\cite{chen2020efficient}, we include a comparison with the Dyadic Block Sketching Linear (DBSL) algorithm~\cite{wen2024matrix}, which represents a more recent adaptive sketching approach.

\begin{algorithm}
\caption{\textit{PSI-LinUCB Rank-1 Update}}
\label{alg:linucb-psi-rank1}
\begin{algorithmic}[1]
\REQUIRE context $x_t$, reward $r_t$, rank $r$
\STATE $b_{t+1} = b_t + r_t \cdot x_t$
\STATE $L_{t}^{-1}=(I - U_{t}\,V_{t}^\top)L_0^{-1}$
\linecomment{We do not form $L_{t}^{-1}$ explicitly}
\STATE $\bar{x}_{t+1} = L_t^{-1} x_t$
\COMMENT{We do not form $L_t^{-1}$ explicitly}
\STATE $\alpha_{t+1} = \frac{\sqrt{1 + \|\bar{x}_{t+1}\|^2} - 1}{\|\bar{x}_{t+1}\|^2}$
\STATE $\beta_{t+1} = \frac{\alpha_{t+1}}{1 + \alpha_{t+1} \|\bar{x}_{t+1}\|^2}$
\IF{$U_t.\text{shape}[1] < 2r$}
    \STATE $U_{t+1} \gets [U_t, \; \beta_{t+1} \bar{x}_{t+1}]$
    \STATE $V_{t+1} \gets [V_t, \; (I - V_t U_t^\top) \bar{x}_{t+1}]$
\ENDIF
\IF{$U_t.\text{shape}[1] = 2r$}
    \IF{first time $U_t.\text{shape}[1] = 2r$}
        \STATE $\tilde{U}_{t+1} S_{t+1} \tilde{V}_{t+1}^\top = SVD(U_{t} V_{t}^\top)$
    \ENDIF
    \STATE $\Delta D_{t+1} = \beta_{t+1} \tilde{x}_{t+1} \tilde{x}_{t+1}^\top (I - U_t V_t^\top)$
    \COMMENT{We do not form $\Delta D_{t+1}$ explicitly}
    \STATE $\tilde{U}_{t+1}, S_{t+1}, \tilde{V}_{t+1} = \textit{PSI}(\tilde{U}_t, S_t, \tilde{V}_t, \Delta D)$
    \STATE $U_{t+1} \;\gets\; \tilde U_{t+1}S_{t+1}$
    \STATE $V_{t+1} \;\gets\; \tilde V_{t+1}$
\ENDIF
\STATE $\theta_{t+1} = L_{t+1}^{-\top} L_{t+1}^{-1} b_{t+1}$
\STATE \textbf{return} $U_{t+1}, V_{t+1}, \theta_{t+1}$
\end{algorithmic}
\end{algorithm}

% \begin{algorithm}
% \caption{\textit{PSI-LinUCB Rank-1 Training}}
% \label{alg:linucb-psi-rank1-training}
% \begin{algorithmic}[1]
% \REQUIRE clusters, time horizon $T$, regularization $\lambda$, exploration $\beta$, rank $r$
% \STATE Initialize environment with clusters, target class
% \STATE Initialize $U_a, V_a, b_a, \theta_a$ for each arm $a \in \mathcal{A}$
% \FOR{$t$ in $\{0, \ldots, T-1\}$}
%     \STATE $\{x_{t,a}\}_{a \in \mathcal{A}} \gets$ get\_contexts()
%     \STATE $a_t \gets \arg\max_{a \in \mathcal{A}} \text{score}(x_{t,a}, a)$
%     \STATE $r_t \gets$ step($a_t$)
%     \STATE $U_{a_t}, V_{a_t}, \theta_{a_t} \gets$ \textit{Update\_arm}$(x_{t,a_t}, r_t, a_t)$
% \ENDFOR
% \STATE \textbf{return} cumulative\_mistakes
% \end{algorithmic}
% \end{algorithm}

% \subsection{Batch size and rank trade-off}\label{appendix:exp6}

% The batch size controls how frequently PSI updates are applied. Larger batches reduce training time by triggering integrator updates less often, but may slightly degrade quality since new data is incorporated less frequently. Figures~\ref{fig:health_batch} and~\ref{fig:beauty_batch} demonstrate this trade-off: training time decreases substantially with larger batches while hit rate remains relatively stable. This enables flexible tuning for online recommendation scenarios: latency-sensitive deployments may favor larger batches, while quality-critical scenarios benefit from smaller batches with more frequent updates.

\begin{figure}[t!]
  \centering
  \begin{subfigure}[b]{0.49\textwidth}
    \includegraphics[width=\textwidth]{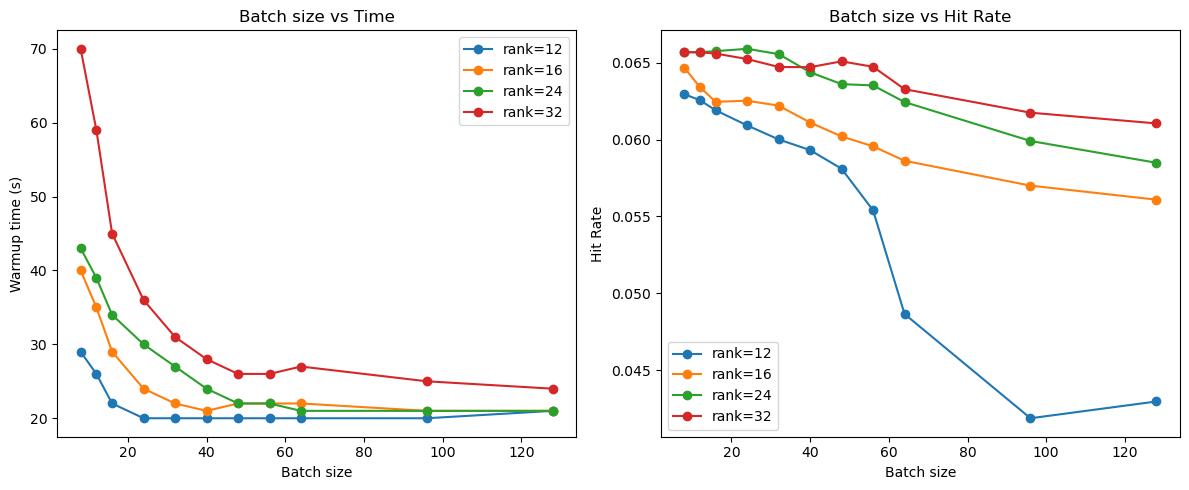}
    \caption{Amazon Health}
    \label{fig:health_batch}
  \end{subfigure}
  \hfill
  \begin{subfigure}[b]{0.49\textwidth}
    \includegraphics[width=\textwidth]{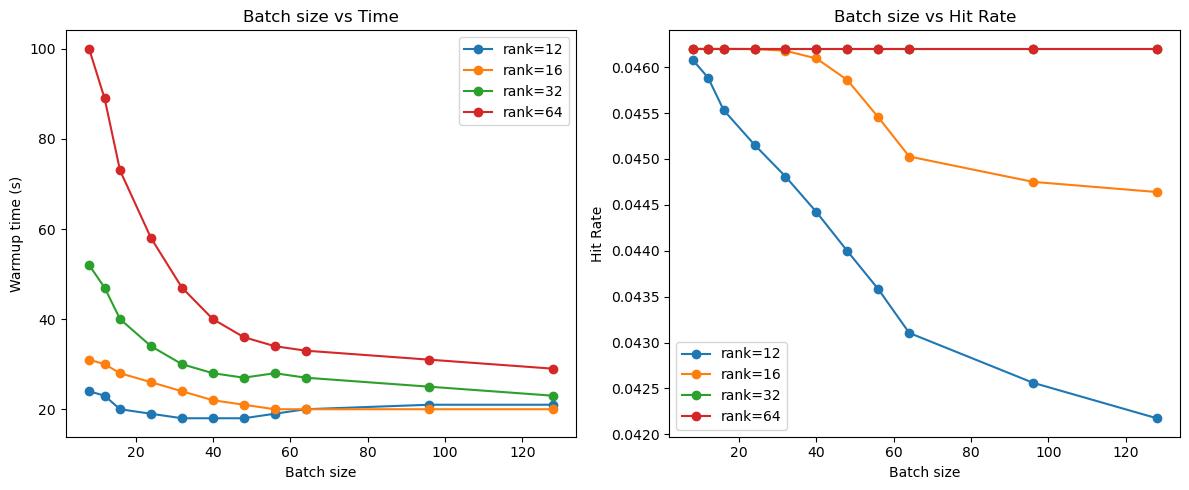}
    \caption{Amazon All Beauty}
    \label{fig:beauty_batch}
  \end{subfigure}
  \caption{Batch/rank dependence on Amazon datasets.}
  \label{fig:batch_dependence_amazon}
\end{figure}

\section{Proofs}
\label{app:proofs}
\subsection{Proof of \Cref{prop: rank-one update}} \label{sec:proof_rank-one update}

% Consider the equation \eqref{eq:cholesky_A_t}. The updated Cholesky factor writes as
% \[
% A_{t+1} = L_t(I+ \alpha_{t+1} \bar {x}_{t+1}\bar{x}_{t+1}^{\top})(I+ \alpha_{t+1} \bar x_{t+1} \bar x_{t+1}^{\top})^{\top}L_t^{\top} = L_{t+1}L_{t+1}^{\top}\eqsp.
% \]
% where the definitions of $\bar x_{t+1}$ and $\alpha_{t+1}$ imply 
% \[
% L_{t+1} = L_t(I+ \alpha_{t+1} \bar {x}_{t+1}\bar{x}_{t+1}^{\top})\eqsp.
% \]
Using \eqref{eq:def_l_t_plus_1_rec}, we obtain 
\[
L_{t+1}^{-1} = (I+ \alpha_{t+1} \bar {x}_{t+1}\bar{x}_{t+1}^{\top})^{-1}L_t^{-1}\eqsp.
\]
After applying the Sherman-Morrison-Woodbury formula
\[
(I+ \alpha_{t+1} \bar {x}_{t+1}\bar{x}_{t+1}^{\top})^{-1} = I- \frac{\alpha_{t+1} \bar {x}_{t+1}\bar{x}_{t+1}^{\top}}{1+\alpha_{t+1}\bar{x}_{t+1}^{\top} \bar {x}_{t+1}} = I - \beta_{t+1}\bar{x}_{t+1}\bar{x}_{t+1}^{\top}\eqsp,
\]
the recursive structure of the inverse update reads:
\[
L_t^{-1} = (I- \beta_{t} \bar {x}_{t}\bar{x}_{t}^{\top})L_{t-1}^{-1}\eqsp.
\]
Unraveling this recursion from $t$ back to the initial condition gives
\begin{align}
L_t^{-1} = (I- \beta_{t} \bar {x}_{t}\bar{x}_{t}^{\top})L_{t-1}^{-1} = \prod_{i=1}^{t}(I- \beta_{i} \bar {x}_{i}\bar{x}_{i}^{\top})L_{0}^{-1}.
\end{align}
We now prove by induction that $L^{-1}_{t} = (I - U_{t}\,V_{t}^\top)L^{-1}_0$. Indeed, for $t=1$, we have
\[
L_{1}^{-1} = (I- \beta_{1} \bar {x}_{1}\bar{x}_{1}^{\top})L_{0}^{-1} = (I - U_1 V_1^{\top})L_{0}^{-1}\eqsp,
\]
where $U_1 = \beta_1 \bar x_1$ and $V_1 = \bar x_1$. Assuming that $L^{-1}_{t} = (I - U_{t}\,V_{t}^\top)L^{-1}_0$, 
\begin{align}
L_{t+1}^{-1} &= (I- \beta_{t+1} \bar {x}_{t+1}\bar{x}_{t+1}^{\top})L_{t}^{-1} \\
&= (I- \beta_{t+1} \bar {x}_{t+1}\bar{x}_{t+1}^{\top})(I - U_{t}\,V_{t}^\top)L^{-1}_0 \\
&= (I-U_tV_t^{\top} -\beta_{t+1} \bar x_{t+1} \bar x_{t+1}^{\top} + \beta_{t+1} \bar x_{t+1} \bar x_{t+1}^{\top}U_tV_t^{\top} )L_0^{-1} \\
&= (I - (U_tV_t^{\top} + \beta_{t+1}\bar x_{t+1} \bar x_{t+1}^{\top}(I - U_tV_t^{\top})))L_0^{-1} \label{eq:L_t_thorough_L_0}\eqsp.
\end{align}

This can be written as $L_{t+1}^{-1} = (I - U_{t+1}V_{t+1}^{\top})L_0^{-1}$, where
\begin{align}
\label{eq:rank1_UV}
U_{t+1} &= \begin{bmatrix} U_t & \beta_{t+1}\,\bar x_{t+1}\end{bmatrix}, \\
V_{t+1} &= \begin{bmatrix} V_t & (I - V_tU_t^{\top})\bar x_{t+1}\end{bmatrix}.
\end{align}
This completes the proof.

\subsection{Proof for \Cref{sec:batch-updates}}
\label{sec:proof_batch update}

% \begin{proof}

% and $\bar{X}_{t+1} = Q_{t+1}R_{t+1}$ and $Y_{t+1}$ can be computed by the theorem \ref{theorem: sym_fact} from \cite{Ambikasaran}(remark 3.4):
% \begin{theorem} \label{theorem: sym_fact}
%     Symmetric factorization of $I+ \bar{X}_{t+1}\bar{X}_{t+1}^{\top}$:
% \begin{itemize}
%     \item $\bar X_{t+1} = Q_{t+1}R_{t+1}$
%     \item $T =I + R_{t+1}\,R_{t+1}^{\top}$  \quad $\in \rset^{B \times B}$
%     \item  Factor $T = M\,M^{\top}$, \quad $M\in\rset^{B\times B}$  
%     \item  $Y_{t+1} =M - I$ $\in \rset^{B \times B}$
    
% \end{itemize}
% \end{theorem}

% \end{proof}

\begin{theorem} \label{theorem:sym_fact}
\label{theorem: sym_fact}
Let $X_{t+1} \in \mathbb{R}^{d \times B}$ be a rank-$B$ update matrix with $B \ll d$. Then the updated matrix $ A_t+  X_{t+1} X_{t+1}^{\top}$ can be symmetrically factored as
\begin{equation}
A_{t+1} = L_{t+1}L_{t+1}^{\top},
\end{equation}
where
\begin{equation}
L_{t+1} = L_t(I + Q_{t+1}Y_{t+1}Q_{t+1}^{\top}),
\end{equation}
and the factors are obtained as follows:

\begin{enumerate}
    \item Compute $\bar{X}_{t+1} = L_t^{-1}X_{t+1} \in \mathbb{R}^{d \times B}$ and perform the QR decomposition
    \begin{equation}
    \bar{X}_{t+1} = Q_{t+1}R_{t+1},
    \end{equation}
    where $Q_{t+1} \in \mathbb{R}^{d \times B}$ is an orthogonal matrix ($Q_{t+1}^{\top}Q_{t+1} = I_B$), and $R_{t+1} \in \mathbb{R}^{B \times B}$ is an upper triangular matrix.
    
    \item  Form the small matrix
    \begin{equation}
    T_{t+1} = I_B + R_{t+1}R_{t+1}^{\top} \in \mathbb{R}^{B \times B}.
    \end{equation}
    
    \item Compute the Cholesky decomposition of the small matrix
    \begin{equation}
    T_{t+1} = M_{t+1}M_{t+1}^{\top},
    \end{equation}
    where $M \in \mathbb{R}^{B \times B}$ is a lower triangular matrix.
    
    \item  Set
    \begin{equation} \label{eq: m}
    Y_{t+1} = M_{t+1} - I_B \in \mathbb{R}^{B \times B}.
    \end{equation}
\end{enumerate}

This result follows from Remark 3.4 in \cite{Ambikasaran}, see equations (3.5)--(3.6).
\end{theorem}

\subsection{Proof of \Cref{prop: batch update_rec}}
\label{sec:proof_batch update_rec}

\begin{proof}
Writing the Cholesky decomposition $A_t = L_t L_t^{\top}$ and factoring out $L_t$ gives
\begin{equation}
\begin{split}
A_{t+1}
&= A_t + X_{t+1}X_{t+1}^{\top} \\[4pt]
&= L_t L_t^{\top} + X_{t+1}X_{t+1}^{\top} \\[4pt]
&= L_t\bigl(I + L_t^{-1}X_{t+1}X_{t+1}^{\top} L_t^{-T}\bigr)L_t^{\top} \\[4pt]
&= L_t\bigl(I + \bar{X}_{t+1}\bar{X}_{t+1}^{\top}\bigr)L_t^{\top},
\end{split}
\end{equation}
where $\bar{X}_{t+1} = L_t^{-1}X_{t+1}$. Applying Theorem~\ref{theorem: sym_fact}, we obtain
\begin{equation}
I + \bar{X}_{t+1}\bar{X}_{t+1}^{\top} = \bigl(I + Q_{t+1}Y_{t+1}Q_{t+1}^{\top}\bigr)\bigl(I + Q_{t+1}Y_{t+1}Q_{t+1}^{\top}\bigr)^{\top},
\end{equation}
which yields
\begin{equation}
A_{t+1} = L_t\bigl(I + Q_{t+1}Y_{t+1}Q_{t+1}^{\top}\bigr)\bigl(I + Q_{t+1}Y_{t+1}Q_{t+1}^{\top}\bigr)^{\top} L_t^{\top} = L_{t+1}L_{t+1}^{\top},
\end{equation}
where
\begin{equation}
L_{t+1} = L_t\bigl(I + Q_{t+1}Y_{t+1}Q_{t+1}^{\top}\bigr).
\end{equation}

The inverse of $L_{t+1}$ is
\begin{equation}
L_{t+1}^{-1} = \bigl(I + Q_{t+1}Y_{t+1}Q_{t+1}^{\top}\bigr)^{-1}L_t^{-1}.
\end{equation}
Applying equation $\eqref{eq: m}$ and the Sherman-Morrison-Woodbury formula gives
\begin{align}
\bigl(I + Q_{t+1}Y_{t+1}Q_{t+1}^{\top}\bigr)^{-1} 
&= I - Q_{t+1}\bigl((M_{t+1} - I_B)^{-1} + I\bigr)^{-1} Q_{t+1}^{\top} \\
&= I - Q_{t+1}(M_{t+1} - I)\bigl(M_{t+1} - I + I\bigr)^{-1} Q_{t+1}^{\top} \\
&= I - Q_{t+1}(M_{t+1} - I)M_{t+1}^{-1} Q_{t+1}^{\top}
\end{align}

We proceed by induction to maintain a low-rank representation of $L_t^{-1}$. Assume that 
\begin{equation}
L_t^{-1} = (I - U_tV_t^{\top})L_0^{-1}.
\end{equation}
Then
\begin{align}
L_{t+1}^{-1} 
&= \bigl(I - Q_{t+1}(M_{t+1} - I)M_{t+1}^{-1} Q_{t+1}^{\top}\bigr)L_t^{-1} \\
&= \bigl(I - Q_{t+1}(M_{t+1} - I)M_{t+1}^{-1} Q_{t+1}^{\top}\bigr)(I - U_tV_t^{\top})L_0^{-1} \\
&= \bigl(I - U_tV_t^{\top} - Q_{t+1}(M_{t+1} - I)M_{t+1}^{-1}Q_{t+1}^{\top}(I - U_tV_t^{\top})\bigr)L_0^{-1} \\
&= (I - U_{t+1}V_{t+1}^{\top})L_0^{-1},
\end{align}
where
\begin{equation}
U_{t+1} = \begin{bmatrix} U_t & Q_{t+1}(M_{t+1} - I)M_{t+1}^{-1}\end{bmatrix} 
\end{equation}
\begin{equation}
V_{t+1} = \begin{bmatrix} V_t & (I - V_tU_t^{\top})Q_{t+1} \end{bmatrix} 
\end{equation}
This completes the inductive step.
\end{proof}

\end{appendix}

\end{document}